\documentclass[10pt]{article} 
\usepackage[accepted]{tmlr}

\usepackage{amsmath,amsfonts,bm}

\def\eqref#1{equation~\ref{#1}}

\def\1{\bm{1}}

\DeclareMathAlphabet{\mathsfit}{\encodingdefault}{\sfdefault}{m}{sl}
\SetMathAlphabet{\mathsfit}{bold}{\encodingdefault}{\sfdefault}{bx}{n}

\usepackage{amsmath}
\usepackage{amssymb}
\usepackage{amsthm}
\usepackage{bm}
\usepackage{booktabs}
\usepackage{multirow}
\usepackage{graphicx}
\usepackage{xcolor}
\usepackage{float}              
\usepackage{enumitem}
\usepackage{algorithm}
\usepackage{tikz}
\usetikzlibrary{positioning,arrows.meta,fit,backgrounds,calc}  
\usepackage{natbib}            
\usepackage{url}
\usepackage{hyperref}
\hypersetup{hidelinks}

\theoremstyle{plain}
\newtheorem{theorem}{Theorem}
\newtheorem{proposition}{Proposition}
\newtheorem{lemma}{Lemma}
\newtheorem{corollary}{Corollary}     
\theoremstyle{definition}
\newtheorem{definition}{Definition}   
\theoremstyle{remark}
\newtheorem{remark}{Remark}

\title{GENERIC-FNO: Embedding Energy Conservation and Entropy Production into Fourier Neural Operators}

\author{\name Jason Sulskis \email jason.sulskis@gtri.gatech.edu \\
\addr Department of Computer Science\\
      University of Illinois at Chicago\\[0.5em]
      \addr Applied Systems Laboratory, Systems Integration Division\\
      Georgia Tech Research Institute
      \AND
      \name Sathya Ravi \email sathya@uic.edu  \\
      \addr Department of Computer Science\\
      University of Illinois at Chicago}

\def\month{11}\def\year{2026}\def\openreview{\url{https://openreview.net/forum?id=gbpEYDdfDB}}

\begin{document}
\maketitle

\begin{abstract}
We propose GENERIC-FNO, a neural operator that places the metriplectic (GENERIC) degeneracy structure of non-equilibrium thermodynamics directly in function space: reversible, energy-conserving dynamics are coupled to irreversible, entropy-producing dynamics via the degeneracy conditions. Structure-preserving neural operators developed previously enforce no more than one conservation law, or else a purely Hamiltonian form, and thermodynamically consistent learning has been restricted to finite-dimensional, graph, or particle systems. In GENERIC-FNO, the energy and entropy functionals are learned as neural operators, while the reversible and irreversible operators take the form of diagonal Fourier multipliers flanked by rank-one projections; these enforce both degeneracy conditions exactly and by construction, requiring no penalty, update projection, or residual. Since the Jacobi identity is not enforced, the structure obtained is metriplectic-degenerate rather than a full Poisson--metriplectic pair. Whatever the initialization, dimension, or resolution, the identities hold to machine precision ($\sim\!10^{-13}$), and therefore the continuous-time dynamics conserve the learned energy and generate the learned entropy exactly, with only an $\mathcal{O}(\Delta t^2)$ drift introduced by explicit time stepping. Such results are guarantees concerning the learned functionals, within the scope of GENERIC's closed conservative--dissipative dynamics; they do not certify physical accuracy, and the $(E,S,L,M)$ decomposition is non-unique. This gauge freedom is made explicit, and we propose a gauge-invariant dissipation diagnostic that does not depend on the learned functionals. The guarantees transfer zero-shot across a $4\times$ super-resolution range and also hold in three dimensions, as tested on three backbones (1D/2D FNO, DeepONet) and four canonical scalar PDEs. The diagnostic identifies, for every backbone, both the reversible system and the most dissipative one. Over $200$-step rollouts, where every unconstrained model we test diverges or collapses, GENERIC-FNO stays bounded, at half the parameters yet $4$--$10\times$ the compute, although accuracy is lost on pure transport and on the smallest 1D backbone.
\end{abstract}

\section{Introduction}
\label{sec:intro}

Neural operators like the Fourier Neural Operator (FNO)~\citep{li2021fourier} and DeepONet~\citep{lu2021learning} have proven to be effective surrogates for the solution operators of partial differential equations (PDEs), capturing maps between function spaces that transfer across different discretizations~\citep{kovachki2023neural}. Yet when deployed autoregressively to generate extended trajectories, these surrogates suffer from error accumulation that goes beyond the purely numerical: energy tends to drift, and quantities meant to decrease monotonically may instead rise spuriously. Because an unconstrained operator possesses no awareness of the conservation and dissipation laws governing the system it emulates, it has no inherent safeguard against leaving the physically admissible manifold.

Efforts to embed physical structure as an inductive bias have expanded considerably. Within function space, however, such structure has thus far been confined to either a single conserved quantity or purely reversible Hamiltonian dynamics—where energy remains constant but no mechanism exists for the irreversible entropy production inherent to any genuine dissipative process. In contrast, research on thermodynamic consistency in machine learning—centered on the GENERIC or metriplectic formalism~\citep{grmela1997dynamics,ottinger2005beyond}, which merges reversible and irreversible dynamics—has been carried out almost entirely in settings involving finite-dimensional state vectors, graphs, or particle systems~\citep{lee2021machine,hernandez2021structure,zhang2022gfinns,gruber2023reversible}. As far as we are aware, no neural operator has yet enforced, in function space, both GENERIC degeneracy conditions: a learned energy and a learned entropy whose reversible and irreversible operators are mutually degenerate. Closing that gap is the aim of this work.

In GENERIC, the dynamics are expressed as $\partial_t u = L\,\delta E/\delta u + M\,\delta S/\delta u$: a reversible bracket, generated by an energy functional $E$ through a skew-adjoint operator $L$, combined with an irreversible bracket, generated by an entropy functional $S$ through a positive semi-definite operator $M$. The first and second laws of thermodynamics emerge from two \emph{degeneracy conditions}, namely $L\,\delta S/\delta u = 0$ and $M\,\delta E/\delta u = 0$, which ensure that the reversible bracket generates no entropy and the irreversible bracket conserves energy. (A complete metriplectic structure additionally requires $L$ to be a Poisson operator, meaning it satisfies the Jacobi identity; however, the degeneracy conditions alone are sufficient to yield the two laws, and these are what we enforce.) This formalism applies to \emph{closed} systems possessing well-defined energy and entropy, whether conservative, dissipative, or a mixture of both; dynamics that are externally forced, open, or non-conservative fall outside its reach and thus beyond the scope of any approach built upon it, including ours. Extending GENERIC to function space introduces challenges not present in the finite-dimensional regime: variational derivatives become densities instead of vectors, operators must behave consistently across varying resolutions and spatial dimensions, and—most importantly—the degeneracy conditions tie the operators to the \emph{state-dependent} gradient directions $\delta E/\delta u$ and $\delta S/\delta u$. Relying on a soft penalty to enforce degeneracy has proven inadequate, as the conditions are only approximately satisfied and the resulting deviations amplify during rollout.

Our approach enforces the degeneracy conditions \emph{exactly, by construction}. We represent $E$ and $S$ as scalar-valued neural operators, while $L$ and $M$ are parameterized as diagonal Fourier multipliers—skew-adjoint and positive semi-definite, respectively, for any parameter values—placed between rank-one projections that eliminate the conjugate gradient direction. Projecting the reversible operator off $\delta S/\delta u$ and the irreversible operator off $\delta E/\delta u$ makes $L\,\delta S/\delta u = 0$ and $M\,\delta E/\delta u = 0$ hold identically. This design requires no penalty term, no projection of the predicted update, and no free residual; consequently, in continuous time, the learned energy is conserved and the learned entropy is produced to machine precision (residuals $\sim\!10^{-13}$), irrespective of initialization, dimension, or resolution, with the explicit update used for rollouts contributing only the standard $\mathcal{O}(\Delta t^2)$ energy drift. Since the multipliers act on a fixed band of Fourier modes and the projections are scale-invariant ratios, the entire operator is independent of resolution and dimension: a model trained on one grid can be evaluated zero-shot on finer grids without losing its guarantees. Two limitations of these guarantees deserve mention from the start. First, the Jacobi identity is not enforced: while the base multiplier is a constant skew operator and thus Poisson by itself, the projection renders $L$ state-dependent, and we neither impose nor verify Jacobi for the projected bracket (\S\ref{sec:limitations}). Second, the guarantees pertain to the \emph{learned} functionals $E_\theta$ and $S_\phi$; they do not, on their own, ensure that the learned flow matches the underlying physics, and we maintain throughout a clear separation between structural consistency and physical accuracy.

This expressiveness brings with it a subtlety. The tuple $(E,S,L,M)$ that realizes a given flow is \emph{not unique}—a recognized gauge freedom of GENERIC~\citep{ottinger2005beyond}. While the structural guarantees are gauge-invariant, the division of the dynamics between reversible and irreversible channels is not, so interpreting a learned model channel by channel demands caution. We do not resolve or characterize this freedom; we prove which of our quantities are invariant under it (Appendix~\ref{app:gauge}) and introduce a gauge-invariant dissipation diagnostic, built on a fixed quadratic energy rather than the learned functionals, which can separate genuinely reversible from dissipative dynamics and be directly compared with ground truth.

For empirical evaluation, we consider three operator backbones (1D and 2D FNOs, and DeepONet) applied to four canonical scalar PDEs on periodic domains, covering reversible (advection), dissipative (heat), and mixed (Burgers, damped wave) dynamics. This provides a controlled testbed with known ground-truth dissipation rates, rather than the complex-fluid scenarios that originally motivated GENERIC. In every configuration, the structural guarantees are met to machine precision, remain intact when transferred zero-shot across a $4\times$ super-resolution range, and hold in three dimensions. The gauge-invariant diagnostic correctly identifies the reversible system and the most dissipative system across all backbones, and on two of the three backbones it also resolves the finer ordering of physical dissipation—an ordinal claim over a limited set of systems, which we present as such. Accuracy outcomes vary by backbone. For the 2D FNO and DeepONet backbones, GENERIC-FNO surpasses unconstrained and energy-penalized baselines on dissipative and mixed problems (in 2D, with half the parameter count) and remains bounded over $200$-step rollouts where every unconstrained model we tried -- including an FNO with matched training compute and a U-Net \citep{ronneberger2015unet} -- diverges or collapses. On pure linear transport within the FNO backbones, it is less accurate, and on the smallest backbone, the 1D FNO, it fails on advection and does not outperform the baselines on heat; a capacity sweep of that backbone (App.~\ref{app:capacity}) reveals that the gap persists across widths from $70$K to $1.1$M parameters, ruling out functional capacity as the cause, with \S\ref{sec:limitations} attributing it largely to normalization inside the functional networks. The exact variational derivatives also incur a notable computational cost: at half the parameters, GENERIC-FNO requires $\approx\!4.4\times$ the inference time, $\approx\!10\times$ the training-step time, and $\approx\!2.3\times$ the peak memory of a plain FNO (App.~\ref{app:repro}).

Our contributions are as follows:
\begin{itemize}
  \item \textbf{A neural operator with both GENERIC degeneracy conditions enforced in function space.} GENERIC-FNO learns the energy and entropy functionals along with the reversible and irreversible operators, jointly yielding degeneracy-consistent dynamics by construction in any spatial dimension, verified in one, two, and three (\S\ref{sec:method}).
  \item \textbf{Degeneracy by construction.} A projection-sandwiched diagonal operator parameterization enforces both degeneracy conditions exactly—without penalty, update projection, or residual—so the structural identities, and the continuous-time conservation of the learned energy and production of the learned entropy they imply, hold to machine precision for any initialization, dimension, or resolution; the explicit rollout adds only an $\mathcal{O}(\Delta t^2)$ drift, and the Jacobi identity is not enforced (\S\ref{sec:operators}, \S\ref{sec:guarantees}, \S\ref{sec:limitations}).
  \item \textbf{Gauge freedom and a gauge-invariant diagnostic.} We note the non-uniqueness of the learned $(E,S,L,M)$ decomposition (without characterizing it), prove which of our claims are gauge-invariant and which are not, and offer a falsifiable dissipation diagnostic that depends on no learned quantity (\S\ref{sec:gauge}, Appendix~\ref{app:gauge}).
  \item \textbf{Resolution- and dimension-independent guarantees.} We
  demonstrate that the exact structural guarantees, not merely the accuracy,
  transfer zero-shot across a $4\times$ super-resolution range in 2D, and that
  they hold at machine precision in 3D, at random initialization and after
  training, on grids of $32^3$ and $64^3$ (\S\ref{sec:superres},
  \S\ref{sec:3d}).
  \item \textbf{A cross-backbone study that reports the failures with the successes.} Across three backbones and four PDEs, we document where thermodynamic structure improves accuracy and where it does not: the 1D-FNO losses, the capacity sweep that excludes capacity as their cause, and the normalization ablation that identifies it; a diagnosed and mitigated coarse-grid limitation for reversible transport; a long-horizon heat case where the learned energy is conserved exactly while the physical energy is under-dissipated, against baselines that include a compute-matched FNO and a U-Net; and a stability analysis of the rank-one projections as the learned gradients vanish, which shows the projections remain bounded and attributes a near-equilibrium update-scale artifact to normalization in the functional networks instead (\S\ref{sec:experiments}, \S\ref{sec:limitations}, App.~\ref{app:stability}).
\end{itemize}
\section{Related Work}
\label{sec:related}

\paragraph{Neural operators.}
Neural operators learn mappings between function spaces and, unlike standard
networks, are designed to be discretization-convergent. The Fourier Neural
Operator~\citep{li2021fourier} parameterizes the integral kernel in the spectral
domain and evaluates across resolutions; DeepONet~\citep{lu2021learning} factorizes
the operator into branch and trunk networks; and a broad subsequent literature
\citep{kovachki2023neural} has extended these to new architectures and domains,
including Laplace~\citep{cao2023laplace} and wavelet~\citep{tripura2023wavelet}
parameterizations of the kernel, nonlinear decoders such as
NOMAD~\citep{seidman2022nomad}, and transformer-based operators such as
UPT~\citep{alkin2024upt} and Transolver++~\citep{luo2025transolverpp}. These
models are flexible black-box surrogates, but none encodes a conservation or
dissipation law, and their autoregressive rollouts can leave the physically
admissible manifold (\S\ref{sec:longhorizon}). We compare against FNO and
DeepONet rather than against these newer architectures for a methodological
reason: GENERIC-FNO is a structural modification that wraps a backbone, so the
controlled test of the thermodynamic prior is a same-backbone comparison
(FNO vs.\ energy-penalized FNO vs.\ GENERIC-FNO; DeepONet vs.\
GENERIC-DeepONet), which holds the architecture fixed and varies only the
constraint. A comparison against a faster or more expressive unconstrained
operator would measure the architecture, not the prior. That ablation answers
whether the prior helps a given backbone, not which surrogate a practitioner
should deploy, and we do not claim the latter. The one place the two questions
meet is long-horizon stability, where a more expressive unconstrained model
might be stable without any prior; \S\ref{sec:longhorizon} therefore also
reports a compute-matched FNO and a U-Net as unconstrained references. The
construction requires of its backbone only a map from a field to a scalar
functional with a differentiable path, which every architecture above
provides; wrapping them is a natural extension we have not carried out.

\paragraph{Structure preservation in finite dimensions.}
A large body of work builds physical structure into learned dynamics for
finite-dimensional state vectors. Hamiltonian and Lagrangian neural networks and
symplectic integrators learn reversible, energy-conserving dynamics
\citep{greydanus2019hamiltonian,cranmer2020lagrangian,jin2020sympnets}, but model
only the conservative half of the picture. The thermodynamically complete
picture is given by the GENERIC (metriplectic)
formalism~\citep{grmela1997dynamics,ottinger2005beyond}, and a line of work
learns its ingredients: structure-preserving neural networks with penalized
degeneracy~\citep{hernandez2021structure}, machine-learned
brackets~\citep{lee2021machine} and GFINNs~\citep{zhang2022gfinns} with
degeneracy by construction, and bracket-based reversible--irreversible layers
for graph networks~\citep{gruber2023reversible}; the related
OnsagerNet~\citep{yu2021onsagernet} learns a single-potential Onsager structure
rather than an energy--entropy pair. A point worth making precise, because it
applies to our construction as well: what the metriplectic methods enforce is
skew-symmetry of $L$, positive semi-definiteness of $M$, and the two degeneracy
conditions. The Jacobi identity that would make a state-dependent $L$ a true
Poisson operator is not enforced in general, and is guaranteed only in the
special case of a state-independent (constant) $L$. The first and second laws
depend on skewness, positivity, and degeneracy alone---not on Jacobi---which is
why this literature, and this paper, can prove energy conservation and entropy
production without it.

Recent work continues in this vein---efficiently parameterized neural
metriplectic systems with by-construction guarantees and
universal-approximation results \citep{gruber2025metriplectic}, and
thermodynamics-informed graph networks with a node-local metriplectic
bias \citep{tierz2025graph}. Two concurrent directions are closest to ours and
worth distinguishing precisely. \citet{baheri2025metriplectic} learn
dissipative dynamics by conditional flow matching with a metriplectic vector
field $J\nabla H_\theta - G_\theta\nabla\Phi_\theta$ and a structure-preserving
Strang--prox sampler, with degeneracy enforced hard or soft; the state, however,
is a finite-dimensional vector (demonstrated on a damped pendulum), and the
extension to PDEs is posed as future work. \citet{oprisa2026metriplector}
repurpose the metriplectic equation as a general-purpose architecture
primitive, evolving learned latent fields for image recognition, constraint
satisfaction, and language modeling; there the dynamics operate on
finite-dimensional latent grids with a canonical symplectic bracket, and the
objective is computation itself rather than a surrogate for physical
dynamics---no learned energy--entropy pair with mutual degeneracy in function
space is involved. All of these methods remain tied to finite-dimensional state
vectors, graphs, or latent grids; their parameterizations of $L$ and $M$ scale
with the state dimension, and they do not confront the function-space
issues---variational derivatives as densities, resolution and dimension
independence---that a field theory requires.

A recurring and important lesson from this literature is that the degeneracy
conditions must be enforced \emph{by construction} rather than through a soft
penalty \citep{lee2021machine,zhang2022gfinns,gruber2025metriplectic}:
penalized models satisfy the conditions only approximately, and the residual
violations accumulate over a rollout.

\paragraph{Structure-preserving neural operators.}
In function space, physical structure has been incorporated mainly as symmetry
or as a single invariant. Physics-informed losses~\citep{raissi2019physics}
penalize the residual of a known equation; conservation-constrained
surrogates enforce a single conserved quantity, softly or by
projection~\citep{hansen2023conservation}; and equivariant models impose
spatial symmetries~\citep{wang2021equivariant,brandstetter2022lie}. To our
knowledge, however, no neural operator enforces the GENERIC degeneracy
conditions---a learned energy and a learned entropy, with reversible and
irreversible operators that are mutually degenerate---in function space.
GENERIC-FNO occupies this gap: it carries the by-construction degeneracy of the
finite-dimensional metriplectic networks into the resolution- and
dimension-independent setting of neural operators, parameterizing $L$ and $M$
as diagonal Fourier multipliers (rather than dimension-scaling matrices) so that
the construction is intrinsically a field-space object. This also separates
GENERIC-FNO from the recent thermodynamics-informed graph and mesh networks
that achieve discretization flexibility~\citep{tierz2025graph}: those carry
the metriplectic structure node-locally on a graph or mesh---a flexible but
still discrete object---whereas our $L$ and $M$ are spectral multipliers on a
fixed band of Fourier modes, making the model a genuine function-space map that
is dimension-agnostic and zero-shot super-resolves the underlying continuous
dynamics (\S\ref{sec:superres}) rather than merely re-meshing.

\paragraph{Gauge freedom in metriplectic models.}
That a given flow admits more than one $(E,S,L,M)$ representation is classical
in the GENERIC literature~\citep{ottinger2005beyond}, and the analysis
literature on gradient and GENERIC structures makes the point sharply: the same
macroscopic equation can carry a whole family of dissipative structures, and
where the equation arises as the limit of a stochastic particle system it is
the large-deviation principle, not the equation, that selects one of
them~\citep{mielke2014gradient,duong2013generic}. The issue has received little
attention in the learning setting, where the learned functionals are often
interpreted directly, and a model trained on macroscopic trajectories has no
fluctuation data to select a representative. We therefore separate what is
gauge-invariant from what is not---structural guarantees are, per-channel
attribution is not---and propose a diagnostic that depends only on the
predicted dynamics and a fixed quadratic energy, so that reversibility and the
rate of physical dissipation can be measured without reference to the
(gauge-dependent) learned functionals.

\paragraph{Thermodynamic consistency as an inductive bias.}
Interest in thermodynamically consistent learning is driven by applications
where violating the second law is not a cosmetic error but a route to
instability and unphysical predictions. Data-driven closure and subgrid models
for turbulence and coarse-grained dynamics can inject or remove energy
spuriously unless dissipation is constrained~\citep{duraisamy2019turbulence};
learned constitutive models for complex and viscoelastic fluids---the setting in
which GENERIC was originally developed~\citep{ottinger2005beyond}---must respect
energy balance and entropy production to remain stable under
deployment~\citep{hernandez2021structure}; and thermodynamics-aware reduced-order
models target long-horizon rollout stability that unconstrained surrogates lack
\citep{lee2021machine}. In all of these, thermodynamic consistency is valued less
as a property to be reported than as a strong, transferable inductive bias that
keeps a learned model on the physical manifold far from its training data. Our
contribution is a construction that makes that bias available, exactly, to
neural operators; we validate it on canonical scalar PDEs with known
ground-truth dissipation, and the closure and complex-fluid problems that
motivate it remain the target of future work (\S\ref{sec:limitations}).

\section{Background: The GENERIC Formalism}
\label{sec:background}

We study an evolving field \(u(x,t)\) on a periodic spatial domain \(\Omega \subset \mathbb{R}^d\). The space carries the \(L^2\) inner product \(\langle f, g\rangle = \int_\Omega f(x)\,g(x)\,dx\), and for any admissible variation \(\eta\), the variational derivative \(\delta F/\delta u\) of a functional \(F\) is defined by \(\frac{d}{d\epsilon}F[u+\epsilon\,\eta]\big|_{\epsilon=0} = \langle \delta F/\delta u,\, \eta\rangle\). Readers unfamiliar with this formalism will find a gentle, self-contained introduction in Appendix~\ref{app:primer}.

The GENERIC framework \citep{grmela1997dynamics,ottinger2005beyond}—the General Equation for the Non-Equilibrium Reversible–Irreversible Coupling—decomposes the time evolution of \(u\) into two contributions, one reversible and one irreversible:
\begin{equation}
  \frac{\partial u}{\partial t}
  \;=\; \underbrace{L\,\frac{\delta E}{\delta u}}_{\text{reversible}}
  \;+\; \underbrace{M\,\frac{\delta S}{\delta u}}_{\text{irreversible}},
  \label{eq:bg-generic}
\end{equation}
where \(E[u]\) denotes total energy, \(S[u]\) entropy, \(L\) a skew-adjoint operator (\(\langle f, Lg\rangle = -\langle Lf, g\rangle\)) that generates reversible Hamiltonian-like motion, and \(M\) a self-adjoint, positive semi-definite operator (\(\langle f, Mf\rangle \ge 0\)) responsible for irreversible friction. An equivalent bracket formulation states that for any observable \(F[u]\), \(\dot F = \{F,E\} + [F,S]\), with the skew Poisson bracket \(\{F,G\} = \langle \delta F/\delta u, L\,\delta G/\delta u\rangle\) and the symmetric, positive semi-definite dissipative bracket \([F,G] = \langle \delta F/\delta u, M\,\delta G/\delta u\rangle\). In its complete form, the Poisson bracket must also obey the Jacobi identity, \(\{\{F,G\},H\}+\{\{G,H\},F\}+\{\{H,F\},G\}=0\), which renders \(L\) a Poisson operator. Yet this property is immaterial for the two thermodynamic laws derived below, which rest only on the skewness of \(L\), the positivity of \(M\), and the degeneracy conditions.

\paragraph{Degeneracy conditions.} The two generators are required to be mutually orthogonal in a specific sense:
\begin{equation}
  L\,\frac{\delta S}{\delta u} = 0,
  \qquad
  M\,\frac{\delta E}{\delta u} = 0.
  \label{eq:bg-degeneracy}
\end{equation}
The first condition ensures that reversible dynamics leave entropy unchanged, while the second guarantees that irreversible dynamics leave energy unchanged. In bracket notation, these become \(\{\,\cdot\,, S\} = 0\) and \([\,\cdot\,, E] = 0\).

\paragraph{The first and second laws.} Thermodynamic consistency emerges directly from conditions~\eqref{eq:bg-degeneracy} combined with the skewness of \(L\) and the positivity of \(M\). Energy is exactly conserved:
\begin{equation}
  \dot E
  = \Big\langle \tfrac{\delta E}{\delta u},\, \dot u\Big\rangle
  = \underbrace{\Big\langle \tfrac{\delta E}{\delta u}, L\,\tfrac{\delta E}{\delta u}\Big\rangle}_{=\,0\ (\text{$L$ skew})}
  + \underbrace{\Big\langle \tfrac{\delta E}{\delta u}, M\,\tfrac{\delta S}{\delta u}\Big\rangle}_{=\,0\ (\text{degeneracy})}
  = 0,
\end{equation}
and entropy can never decrease:
\begin{equation}
  \dot S
  = \Big\langle \tfrac{\delta S}{\delta u},\, \dot u\Big\rangle
  = \underbrace{\Big\langle \tfrac{\delta S}{\delta u}, L\,\tfrac{\delta E}{\delta u}\Big\rangle}_{=\,0\ (\text{degeneracy})}
  + \Big\langle \tfrac{\delta S}{\delta u}, M\,\tfrac{\delta S}{\delta u}\Big\rangle
  \;\ge\; 0,
\end{equation}
with the final inequality following from the positive semi-definiteness of \(M\). Hence the first law is satisfied exactly, and the second law holds with entropy production rate \([S,S] = \langle \delta S/\delta u, M\,\delta S/\delta u\rangle \ge 0\).

\paragraph{Limiting cases.} Between purely reversible and purely irreversible behavior, the formalism spans a full spectrum. Setting \(M = 0\) yields Hamiltonian evolution that conserves energy—linear advection, which transports \(u\) while preserving \(\tfrac12\lVert u\rVert^2\), serves as an example. Setting \(L = 0\) produces a pure gradient flow that dissipates monotonically, as seen in diffusion. Truly coupled systems, by contrast, activate both generators at once. The PDEs examined in \S\ref{sec:experiments} are selected to cover this range.

\paragraph{Non-uniqueness.} One further feature is that the representation is not unique: different tuples \((E,S,L,M)\) can give rise to the same flow \(\partial_t u\) \citep{ottinger2005beyond}. This gauge freedom does not affect the guarantees established above, which depend solely on the flow itself and the algebraic conditions~\eqref{eq:bg-degeneracy}. It does, however, have direct consequences for interpreting a learned decomposition, an issue we take up in \S\ref{sec:gauge}.
\section{Method: GENERIC-FNO}
\label{sec:method}

We formulate the GENERIC dynamics from Section~\ref{sec:background} (\eqref{eq:bg-generic}) using two trainable scalar functionals, $E_\theta[u]$ and $S_\phi[u]$, together with operators $L$ and $M$ that fulfill the degeneracy requirements of Eq.~\eqref{eq:bg-degeneracy} identically through their design. Neither functional receives supervision toward a predetermined energy or entropy; the data-fitting loss alone, acting through the dynamics, shapes $E_\theta$ and $S_\phi$. Thermodynamic consistency thus arises from the architectural construction rather than from any added penalty term. The forward map is depicted in Figure~\ref{fig:arch}.

\subsection{Learned energy and entropy functionals}
\label{sec:functionals}

Each functional takes the input field to a single real number, $E_\theta, S_\phi : u \mapsto \mathbb{R}$. Our implementation employs a neural-operator backbone---a Fourier neural operator~\citep{li2021fourier} in the primary model and a DeepONet~\citep{lu2021learning} in the appendix investigation---that generates a feature density field. This field undergoes spatial averaging followed by processing through a compact MLP head:
\begin{equation}
  E_\theta[u] = h_\theta\!\Big(\tfrac{1}{|\Omega|}\!\int_\Omega g_\theta(u)\,dx\Big),
  \qquad
  S_\phi[u] = h_\phi\!\Big(\tfrac{1}{|\Omega|}\!\int_\Omega g_\phi(u)\,dx\Big),
\end{equation}
where $g_{(\cdot)}$ indicates the backbone and $h_{(\cdot)}$ a two-layer head. Resolution independence stems from the spatial \emph{average} rather than a fixed flattening procedure. Reverse-mode automatic differentiation through the field yields the variational derivatives:
\begin{equation}
  \frac{\delta E}{\delta u} = \nabla_u E_\theta(u),
  \qquad
  \frac{\delta S}{\delta u} = \nabla_u S_\phi(u),
\end{equation}
computed with \texttt{create\_graph=True} so that differentiability propagates through the entire dynamics. Instance normalization appears within the FNO backbone blocks, and this choice matters: its absence reduces the variational derivatives by roughly $200\times$ at initialization, and at the 2D resolution under study, the functional networks fail to converge (Appendix~\ref{app:underdiss}). Section~\ref{sec:longhorizon} diagnoses the trade-off this normalization introduces.

\subsection{Structure-preserving operators by construction}
\label{sec:operators}

We choose $L$ and $M$ as diagonal Fourier multipliers restricted to the lowest $m$ modes, a formulation that works uniformly across spatial dimensions and grid configurations. With $\widehat{(\cdot)}$ representing the real FFT, the base multipliers take the form
\begin{equation}
  \widehat{D_L}(k) = \mathrm{i}\,a(k), \qquad
  \widehat{D_M}(k) = |b(k)|^2 \ge 0,
\end{equation}
where $a(k)$ is learnable and real while $b(k)$ is learnable and complex. Hermitian symmetry enforced by the inverse real FFT guarantees that $D_L$ is precisely skew-adjoint and $D_M$ precisely self-adjoint positive semi-definite on real fields, irrespective of parameter values.

\paragraph{Degeneracy by construction.}
A standalone multiplier cannot satisfy the degeneracy conditions of Eq.~\eqref{eq:bg-degeneracy}, since those conditions couple the operators to the \emph{state-dependent} directions $\delta E/\delta u$ and $\delta S/\delta u$. Extending the by-construction methodology from finite-dimensional metriplectic networks~\citep{lee2021machine,zhang2022gfinns,gruber2023reversible} to function space, we insert each multiplier between rank-one $L^2$ projections that eliminate the conjugate direction. Specifically,
\begin{equation}
  P_E\,v = \frac{\langle \delta E/\delta u,\, v\rangle}
                {\langle \delta E/\delta u,\, \delta E/\delta u\rangle}\,
           \frac{\delta E}{\delta u},
  \qquad
  P_S\,v = \frac{\langle \delta S/\delta u,\, v\rangle}
                {\langle \delta S/\delta u,\, \delta S/\delta u\rangle}\,
           \frac{\delta S}{\delta u},
  \label{eq:projections}
\end{equation}
and we set
\begin{equation}
  L = (I-P_S)\,D_L\,(I-P_S), \qquad
  M = (I-P_E)\,D_M\,(I-P_E).
  \label{eq:construction}
\end{equation}
Because $(I-P_S)\,\delta S/\delta u = 0$ and $(I-P_E)\,\delta E/\delta u = 0$, the degeneracy conditions hold as identities: $L\,\delta S/\delta u = 0$ and $M\,\delta E/\delta u = 0$. The projections act as symmetric idempotents, so sandwiching preserves the multiplier structure---$L$ stays skew-adjoint and $M$ stays positive semi-definite. Consequently, the reversible operator is stripped of the entropy-gradient component and the irreversible operator of the energy-gradient component, with those directions supplied (dashed in Fig.~\ref{fig:arch}) by the \emph{opposite} functional. The field update
\begin{equation}
  u_{t+1} = u_t + \big(\,L\,\tfrac{\delta E}{\delta u} + M\,\tfrac{\delta S}{\delta u}\,\big)
  \label{eq:update}
\end{equation}
therefore obeys the first and second laws by design, without any penalty, update projection, or leftover residual.

\paragraph{Scope of the construction.}
Equation~\eqref{eq:construction} ensures skew-adjointness of $L$, positive semi-definiteness of $M$, and both degeneracy conditions for all parameters and states. It does \emph{not} ensure the Jacobi identity. The base multiplier $D_L$ is a constant skew operator and thus Poisson in isolation, but the projection $I-P_S$ introduces $u$-dependence through $\delta S/\delta u$, making the projected bracket $\{F,G\}=\langle \delta F/\delta u, L(u)\,\delta G/\delta u\rangle$ state-dependent; we neither enforce nor check Jacobi for it. This matches the approach of the finite-dimensional learned-bracket methods we draw upon (Section~\ref{sec:related}) and leaves the guarantees proved below unaffected: Theorem~\ref{thm:laws} relies solely on skewness, positivity, and degeneracy. The reversible part loses its interpretation as a Hamiltonian flow on a Poisson manifold, a point we take up in Section~\ref{sec:limitations}.

\paragraph{Behavior near vanishing gradients.}
As $E_\theta$ or $S_\phi$ approaches a critical point, the denominators in Eq.~\eqref{eq:projections} tend toward zero. In practice, we compute $P_w v = \langle w,v\rangle\,w/(\langle w,w\rangle+\varepsilon)$ with $\varepsilon=10^{-12}$ (single precision). The exact degeneracy identities follow from the ideal orthogonal projections; regularization introduces a small residual, verified to machine precision over the evaluated regimes (Appendix~\ref{app:stability}). This limit remains well behaved for two reasons. First, Cauchy--Schwarz gives $\lVert (I-P_w)v\rVert \le \lVert v\rVert$ for any $w$, so the projection acts as a contraction without amplification; when $\lVert w\rVert^2 \ll \varepsilon$, it smoothly approaches the identity, which is harmless since the direction being removed is itself negligible and the absolute degeneracy residuals remain small in that regime. Second, the update $L\,\delta E/\delta u + M\,\delta S/\delta u$ depends linearly on the gradients and thus vanishes alongside them. Appendix~\ref{app:stability} confirms this behavior at the operator level across fourteen orders of magnitude in $\lVert w\rVert$, within the model when gradients are scaled toward zero (the update remains exactly proportional down to $10^{-9}$, with absolute degeneracy residuals at or below $10^{-14}$ and no NaN or Inf), and over a $2000$-step rollout approaching equilibrium. One caveat involves normalization: once $\lVert\delta E/\delta u\rVert^2<\varepsilon$, the dimensionless ratio $r_E$ loses its meaning, though in that regime both the update and the absolute energy drift are already insignificant ($\lesssim 10^{-14}$). These probes also exposed a distinct effect unrelated to the projections: instance normalization within the functional backbones causes $\lVert\delta E/\delta u\rVert$ to scale as $1/\lVert u\rVert$ near the zero field, so the update magnitude does not diminish with the field size close to equilibrium. Section~\ref{sec:limitations} addresses the implications.

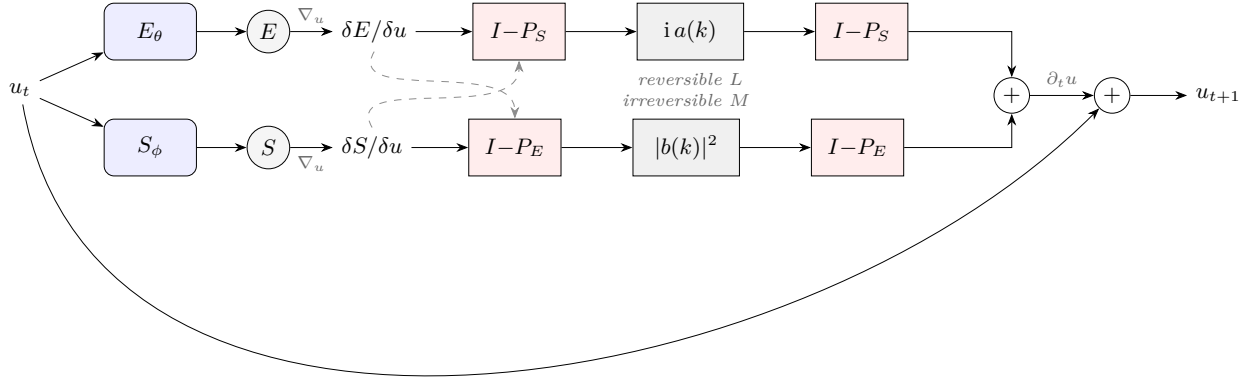
\begin{figure}[t]
\centering
\resizebox{\linewidth}{!}{\begin{tikzpicture}[
  >=Stealth, font=\small, node distance=7mm and 10mm,
  net/.style ={draw, rounded corners, minimum height=8mm, minimum width=13mm, fill=blue!7},
  op/.style  ={draw, minimum height=8mm, minimum width=15mm, fill=black!6},
  proj/.style={draw, minimum height=8mm, minimum width=13mm, fill=red!7},
  sc/.style  ={draw, circle, minimum size=6mm, fill=black!4, inner sep=1pt},
  sm/.style  ={draw, circle, minimum size=5mm, inner sep=0pt},
  sub/.style ={font=\scriptsize\itshape, text=black!60}
]
\node (u) {$u_t$};
\node[net, above right=2mm and 9mm of u] (En) {$E_\theta$};
\node[net, below right=2mm and 9mm of u] (Sn) {$S_\phi$};
\node[sc, right=7mm of En] (E) {$E$};
\node[sc, right=7mm of Sn] (S) {$S$};
\node[right=6mm of E] (dE) {$\delta E/\delta u$};
\node[right=6mm of S] (dS) {$\delta S/\delta u$};
\node[proj, right=8mm of dE] (PS1) {$I{-}P_S$};
\node[op,   right=of PS1]    (DL)  {$\mathrm{i}\,a(k)$};
\node[proj, right=of DL]     (PS2) {$I{-}P_S$};
\node[sub, below=0.5mm of DL] {reversible $L$};
\node[proj, right=8mm of dS] (PE1) {$I{-}P_E$};
\node[op,   right=of PE1]    (DM)  {$|b(k)|^2$};
\node[proj, right=of DM]     (PE2) {$I{-}P_E$};
\node[sub, above=0.5mm of DM] {irreversible $M$};
\node[sm, right=12mm of PS2, yshift=-9mm] (plus) {$+$};
\node[sm, right=9mm of plus] (step) {$+$};
\node[right=8mm of step] (out) {$u_{t+1}$};
\draw[->] (u) -- (En); \draw[->] (u) -- (Sn);
\draw[->] (En) -- (E); \draw[->] (Sn) -- (S);
\draw[->] (E) -- (dE) node[midway, above, sub] {$\nabla_u$};
\draw[->] (S) -- (dS) node[midway, below, sub] {$\nabla_u$};
\draw[->] (dE) -- (PS1); \draw[->] (PS1) -- (DL); \draw[->] (DL) -- (PS2);
\draw[->] (dS) -- (PE1); \draw[->] (PE1) -- (DM); \draw[->] (DM) -- (PE2);
\draw[->] (PS2.east) -| (plus);
\draw[->] (PE2.east) -| (plus);
\draw[->] (plus) -- (step) node[midway, above, sub] {$\partial_t u$};
\draw[->] (step) -- (out);
\draw[->] (u) to[out=-80, in=-135] (step);
\draw[->, dashed, black!45] (dS.north) to[out=120, in=-90] (PS1.south);
\draw[->, dashed, black!45] (dE.south) to[out=-120, in=90] (PE1.north);
\end{tikzpicture}}
\caption{\textbf{GENERIC-FNO forward map.} Two neural operators produce scalar
functionals $E_\theta[u]$ and $S_\phi[u]$, whose variational derivatives are
taken by autodiff. Each diagonal Fourier multiplier---skew $\mathrm{i}\,a(k)$ for
the reversible operator $L$, positive semi-definite $|b(k)|^2$ for the
irreversible operator $M$---is sandwiched between the rank-one projection that
removes the conjugate gradient (dashed: $\delta S/\delta u$ feeds the $L$ branch,
$\delta E/\delta u$ feeds the $M$ branch). This makes $L\,\delta S/\delta u=0$ and
$M\,\delta E/\delta u=0$ hold identically, so the associated continuous-time dynamics conserve the learned energy and
produce the learned entropy by construction; the displayed explicit update
introduces time-discretization error.}
\label{fig:arch}
\end{figure}

\subsection{Exact guarantees and machine-precision verification}
\label{sec:guarantees}

Independently of training, initialization, dimension, or resolution, the construction in \eqref{eq:construction} forces the GENERIC degeneracy structure to hold to machine precision. At random initialization on a \(16\times16\) grid, the following measurements are taken (maximum over a batch; the same identities in 3D are in Table~\ref{tab:3d-identities}):

\begin{center}
\begin{tabular}{lcl}
  \toprule
  Property & Identity & Measured \\
  \midrule
  Reversible skewness & $\langle \delta E/\delta u,\,L\,\delta E/\delta u\rangle = 0$ & $7\!\times\!10^{-19}$ \\
  Energy degeneracy   & $\langle \delta E/\delta u,\,M\,\delta S/\delta u\rangle = 0$ & $3\!\times\!10^{-13}$ \\
  Entropy degeneracy  & $\langle \delta S/\delta u,\,L\,\delta E/\delta u\rangle = 0$ & $1\!\times\!10^{-15}$ \\
  Energy conservation & $\langle \delta E/\delta u,\,\partial_t u\rangle = 0$         & $3\!\times\!10^{-13}$ \\
  Entropy production  & $\langle \delta S/\delta u,\,\partial_t u\rangle \ge 0$       & $\ge +7\!\times\!10^{-9}$ \\
  Operator positivity & $\langle v,\,M v\rangle \ge 0$                                & $\ge +48.5$ \\
  \bottomrule
\end{tabular}
\end{center}

In the semi-discrete (continuous-time) setting, the learned energy is conserved exactly: \(\langle \delta E/\delta u,\partial_t u\rangle = 0\) holds to machine precision. Only the standard \(\mathcal{O}(\Delta t^2)\) integration error arises from the explicit update in \eqref{eq:update}, and we report this per step in \S\ref{sec:experiments}, returning to it in \S\ref{sec:limitations}. Three distinct notions deserve separation. (i)~The \emph{algebraic degeneracy residuals} listed above vanish to machine precision by construction, regardless of state or resolution. (ii)~The \emph{discrete energy drift} introduced by explicit time stepping is generally nonzero and has local order \(\mathcal{O}(\Delta t^2)\). Separately, the normalized structural residual \(r_E\) measures the alignment of the energy gradient with the predicted update; our measurements give \(r_E\!\sim\!10^{-6}\) (\S\ref{sec:experiments}, \S\ref{sec:limitations}). (iii)~What is conserved is the \emph{learned} energy \(E_\theta\), which, like \(S_\phi\), is determined only up to the gauge of \S\ref{sec:gauge}. Conserving \(E_\theta\) by itself implies nothing about the \emph{physical} energy \(Q[u]=\tfrac12\lVert u\rVert^2\): the construction allows \(Q\) to be under-dissipated or even to grow along a rollout while every identity above remains true to machine precision (\S\ref{sec:longhorizon}). Hence the physically meaningful, gauge-invariant quantity is the dissipation diagnostic \(r_{\mathrm{mech}}\), not the value of \(E_\theta\) itself. While (i) is guaranteed exactly by the construction, (ii) and (iii) are what the experiments quantify.

\subsection{Gauge freedom and a gauge-invariant dissipation measure}
\label{sec:gauge}

Point (iii) of \S\ref{sec:guarantees} is now made precise: the construction yields an \emph{exact} thermodynamic structure, but neither energy nor entropy is \emph{unique}. As already noted in \S\ref{sec:background}, the tuple \((E,S,L,M)\) that realizes a given flow is not unique~\citep{ottinger2005beyond}; consequently, the learned \(E_\theta,S_\phi\)—and any quantity derived from them—are fixed only up to this gauge. The exact structural identities of \S\ref{sec:guarantees} are gauge-invariant, depending solely on the flow and the algebraic degeneracy conditions; the relative weighting of the reversible and irreversible channels, however, is not. This motivates separating two types of diagnostics. The \emph{mechanism} is captured by the (gauge-dependent) dissipative fraction \(\rho_M = \lVert M\,\delta S/\delta u\rVert / (\lVert L\,\delta E/\delta u\rVert + \lVert M\,\delta S/\delta u\rVert)\) and the normalized learned-entropy production \(r_S = \langle \delta S/\delta u,\partial_t u\rangle / (\lVert \delta S/\delta u\rVert\,\lVert \partial_t u\rVert)\), both reported only in the appendix. The \emph{thermodynamic behavior} reported in the main text is gauge-invariant: using the fixed quadratic energy \(Q[u]=\tfrac12\lVert u\rVert^2\), which involves no learned quantity, we measure the normalized dissipation rate

\begin{equation}
  r_{\mathrm{mech}} = \frac{-\langle u,\partial_t u\rangle}
                           {\lVert u\rVert\,\lVert \partial_t u\rVert} \in [-1,1],
  \label{eq:rmech}
\end{equation}

which is \(>0\) for genuine dissipation of physical energy and \(\approx 0\) for reversible dynamics. As a ground-truth reference, we compute the analogous relative rate \(\Pi^\star = (Q[u]-Q[u^\star_{t+1}])/Q[u]\) from the reference trajectory, which depends only on the data. Since \(r_{\mathrm{mech}}\) is bounded and unaffected by single-step prediction error, it remains reliable where pointwise accuracy deteriorates. The reversible case provides the decisive test of thermodynamic consistency: a faithful operator must introduce \emph{no} spurious dissipation when \(\Pi^\star\!\approx\!0\), and this test is gauge-invariant because reversible data admits no positive entropy-production representation. In \S\ref{sec:experiments}, we demonstrate that \(|r_{\mathrm{mech}}|\) nearly vanishes on the reversible benchmark and that heat is the most dissipative system in every backbone. The reversible test functions as a pass/fail criterion and bears the weight of the argument; the finer ordering is a rank statement over a handful of systems per backbone, reported with that limitation acknowledged.

\subsection{Discretization, resolution independence, and complexity}
\label{sec:discretization}

All operators operate via the FFT and the \(L^2\) inner product on the grid. The multipliers \(a(k),b(k)\) are defined on a fixed band of \(m\) low-frequency modes and applied to the corresponding modes of whatever grid the input resides on; the projections are scale-invariant ratios. As a result, the entire map is independent of resolution and dimension: a model trained at one grid can evaluate zero-shot at finer grids, and the same code runs unchanged in 1D, 2D, and 3D. Each projection requires two inner products and a scaled subtraction (\(\mathcal{O}(N)\) for \(N\) grid points), which is negligible compared to the FFTs. The method's real cost lies elsewhere: computing the variational derivatives \(\delta E/\delta u\) and \(\delta S/\delta u\) requires two reverse-mode passes per step, and training backpropagates through them (a double-backward). At \(128\times128\), this makes inference \(\approx4.4\times\), a training step \(\approx10\times\), and peak training memory \(\approx2.3\times\) that of a plain FNO, while using half the parameter count (Appendix~\ref{app:repro}, Table~\ref{tab:compute}). This overhead is inherent to computing exact variational derivatives rather than an implementation artifact.
\section{Experiments}
\label{sec:experiments}

\subsection{Setup}
\label{sec:setup}

\paragraph{PDEs.}
Our evaluation covers four PDEs spanning the reversible–irreversible spectrum introduced in \S\ref{sec:background}. The heat equation, a purely dissipative gradient flow ($L\!\to\!0$), is given by
\begin{equation}
  \partial_t u = \nu\,\nabla^2 u,
  \label{eq:heat}
\end{equation}
and monotonically dissipates $\tfrac12\lVert u\rVert^2$. In contrast, linear advection is reversible ($M\!\to\!0$),
\begin{equation}
  \partial_t u + c\,(\partial_x u + \partial_y u) = 0,
  \label{eq:adv}
\end{equation}
which transports $u$ while conserving $\tfrac12\lVert u\rVert^2$ exactly, serving as a clean, fully observed reversible test case. The viscous Burgers equation mixes nonlinear transport with diffusion,
\begin{equation}
  \partial_t u + u\,(\partial_x u + \partial_y u) = \nu\,\nabla^2 u .
  \label{eq:burg}
\end{equation}
We present the 2D forms here; the 1D analogues used in the appendix are identical except for a single spatial derivative. The fourth PDE, a damped wave, is second-order in time and is handled separately in \S\ref{sec:accuracy}.

\paragraph{Data, backbones, and baselines.}
Trajectories are generated spectrally from initial conditions that are band-limited random fields. We examine three operator backbones: a 2D FNO (our main model), a 1D FNO, and a 1D DeepONet. For the FNO backbones, GENERIC-FNO is compared against an unconstrained FNO and an energy-penalized variant (EP-FNO) that incorporates a soft energy-conservation loss; for DeepONet, GENERIC-DeepONet is compared with an unconstrained DeepONet. GENERIC models receive no supervision on $E$ or $S$. All results are reported as mean~$\pm$~std over five seeds, with both data and initialization resampled each time. Since five seeds provide an indicative rather than precise spread, we also note, for conclusions resting on a mean difference, how many seeds support it (models trained under the same seed share data but not initialization, Appendix~\ref{app:repro}), and we rely on standard deviations only when the effect is large relative to them.

\paragraph{Metrics.}
We report relative $L^2$ error for both one-step and 10-step autoregressive rollouts. For GENERIC models, we additionally report the structural residual $r_E = |\langle \delta E/\delta u, \partial_t u\rangle| / (\lVert\delta E/\delta u\rVert\,\lVert\partial_t u\rVert)$ (zero for a vanishing continuous-time energy derivative) and the gauge-invariant dissipation rate $r_{\mathrm{mech}}$ from Eq.~\eqref{eq:rmech}, comparing these against the ground-truth rate $\Pi^\star$.

\subsection{Predictive accuracy}
\label{sec:accuracy}

\begin{table}[ht]
\centering
\caption{\textbf{2D FNO, five-seed accuracy}, with the persistence baseline
($u_{t+1}=u_t$). GENERIC-FNO uses fewer than half the parameters of the learned
baselines. Bold = best learned model. On this short horizon the dynamics are
slow: persistence is within noise of the model of record on heat, better than
every learned model on Burgers, and beaten only by FNO on advection; the
discriminating comparison is the $200$-step rollout of \S\ref{sec:longhorizon},
where persistence fails. Among the learned models GENERIC beats FNO on heat in
three of five seeds and EP-FNO in five, wins Burgers (and the damped wave,
reported in prose) in every seed, and loses advection in every seed.}
\label{tab:acc-2d}
\begin{tabular}{llccc}
\toprule
PDE & Model & Params & $L^2$ (1-step) & $L^2$ (rollout) \\
\midrule
\multirow{4}{*}{Heat}
 & Persistence & -- & $0.017$ & $0.097$ \\
 & FNO         & 2.10M & $0.019\pm0.005$ & $0.107\pm0.029$ \\
 & EP-FNO      & 2.10M & $0.021\pm0.003$ & $0.118\pm0.015$ \\
 & GENERIC-FNO & 1.00M & $\mathbf{0.018\pm0.001}$ & $\mathbf{0.092\pm0.003}$ \\
\midrule
\multirow{4}{*}{Advection}
 & Persistence & -- & $0.032$ & $0.177$ \\
 & FNO         & 2.10M & $\mathbf{0.023\pm0.002}$ & $\mathbf{0.125\pm0.011}$ \\
 & EP-FNO      & 2.10M & $0.024\pm0.002$ & $0.129\pm0.009$ \\
 & GENERIC-FNO & 1.00M & $0.046\pm0.009$ & $0.198\pm0.011$ \\
\midrule
\multirow{4}{*}{Burgers}
 & Persistence & -- & $0.003$ & $0.016$ \\
 & FNO         & 2.10M & $0.029\pm0.007$ & $0.154\pm0.033$ \\
 & EP-FNO      & 2.10M & $0.025\pm0.002$ & $0.135\pm0.012$ \\
 & GENERIC-FNO & 1.00M & $\mathbf{0.005\pm0.001}$ & $\mathbf{0.028\pm0.004}$ \\
\bottomrule
\end{tabular}
\end{table}

\begin{table}[ht]
\centering
\caption{\textbf{1D DeepONet, five-seed accuracy}, with the persistence baseline.
The unconstrained DeepONet is indistinguishable from persistence on every PDE
(and remains so when grown to $175$K parameters, Appendix~\ref{app:donmatch}):
it does not learn to advance the field. GENERIC-DeepONet does, beating
persistence by $8\times$ on heat and $9\times$ on advection. It carries the
extra $E,S$ branch, so the parameter-matched ablation of
Appendix~\ref{app:donmatch} controls for size; the DeepONet functionals carry
no normalization, and GENERIC-DeepONet is the only model in
Tables~\ref{tab:acc-2d}--\ref{tab:acc-don} that beats persistence on heat by nearly
an order of magnitude (\S\ref{sec:limitations}).}
\label{tab:acc-don}
\begin{tabular}{llccc}
\toprule
PDE & Model & Params & $L^2$ (1-step) & $L^2$ (rollout) \\
\midrule
\multirow{3}{*}{Heat}
 & Persistence      & --   & $0.010$ & $0.054$ \\
 & DeepONet         & 58K  & $0.010\pm0.001$ & $0.055\pm0.004$ \\
 & GENERIC-DeepONet & 175K & $\mathbf{0.001\pm0.000}$ & $\mathbf{0.007\pm0.001}$ \\
\midrule
\multirow{3}{*}{Advection}
 & Persistence      & --   & $0.037$ & $0.201$ \\
 & DeepONet         & 58K  & $0.036\pm0.002$ & $0.197\pm0.010$ \\
 & GENERIC-DeepONet & 175K & $\mathbf{0.004\pm0.001}$ & $\mathbf{0.022\pm0.006}$ \\
\midrule
\multirow{3}{*}{Burgers}
 & Persistence      & --   & $0.004$ & $0.019$ \\
 & DeepONet         & 58K  & $0.004\pm0.000$ & $0.019\pm0.002$ \\
 & GENERIC-DeepONet & 175K & $\mathbf{0.003\pm0.000}$ & $\mathbf{0.016\pm0.001}$ \\
\bottomrule
\end{tabular}
\end{table}

\begin{figure}[t]
\centering
\includegraphics[width=\linewidth]{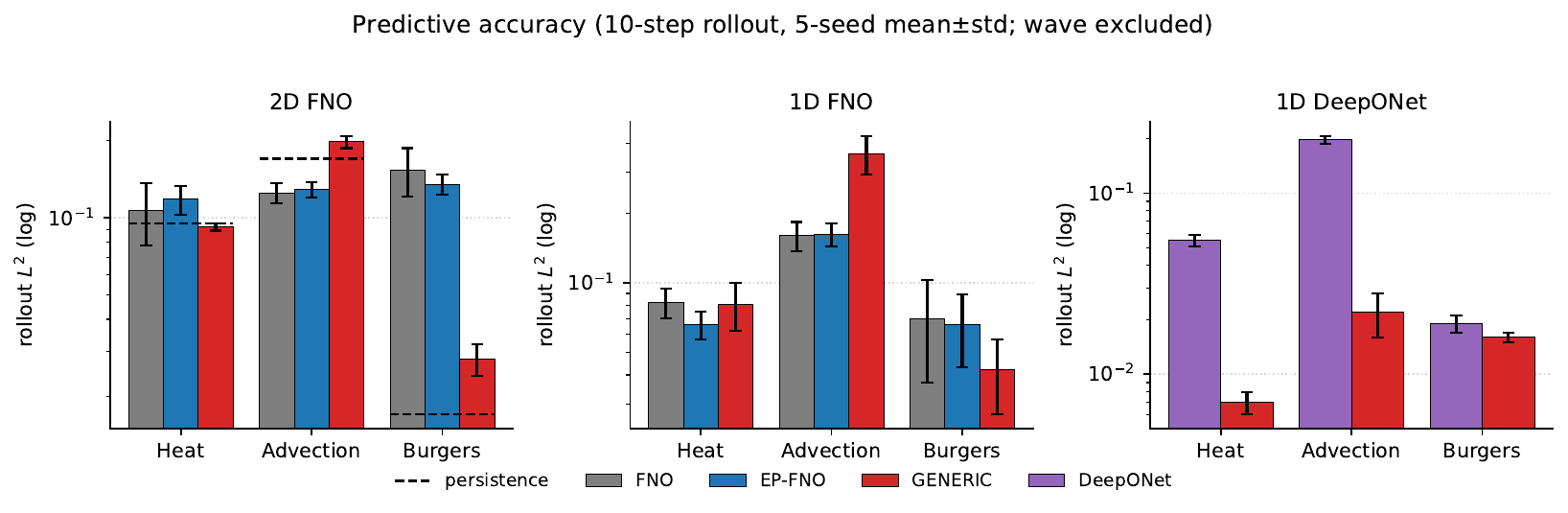}
\caption{\textbf{Rollout $L^2$ error with five-seed error bars, one panel per backbone
(2D FNO, 1D FNO, 1D DeepONet), wave excluded.} GENERIC (red) versus
unconstrained and energy-penalized baselines; dashed ticks mark the persistence
baseline where measured. Visual companion to
Tables~\ref{tab:acc-2d}--\ref{tab:acc-don}; bars are log-scaled.}
\label{fig:accuracy}
\end{figure}

On the 2D FNO backbone (Table~\ref{tab:acc-2d}, Fig.~\ref{fig:accuracy}), GENERIC-FNO surpasses both learned baselines on the dissipative and mixed problems while employing fewer than half the parameters, although two of these gains require scrutiny against the persistence row. In the Burgers case, rollout error declines from $0.154$ (FNO) and $0.135$ (EP-FNO) to $0.028$ consistently across every seed; persistence registers $0.017$, so each baseline performs worse than simply leaving the field untouched, whereas GENERIC-FNO approaches that reference level. For heat, GENERIC-FNO outperforms FNO in three of five seeds and EP-FNO in all five, with a variance an order of magnitude tighter ($0.092\pm0.003$ vs.\ $0.107\pm0.029$); the average difference stems from one outlier FNO seed, and persistence ($0.095$) falls within the noise of all three learned models, making across-seed stability the defensible conclusion rather than any accuracy improvement. Because the time step here is short, the field evolves minimally over ten steps, so these tables effectively gauge how well each model avoids degrading a slowly varying field; the comparison that genuinely separates the models is the $200$-step rollout of \S\ref{sec:longhorizon}, where the true field sheds $86\%$ of its energy, persistence collapses, the baselines blow up, and GENERIC-FNO remains contained. Additionally, on heat GENERIC-FNO's physical-energy tracking error runs $2$--$3\times$ higher than FNO's even though its $L^2$ is lower, and it raises $Q=\tfrac12\|u\|^2$ on as many as $13\%$ of rollout steps (FNO: up to $10\%$, in a single seed); \S\ref{sec:longhorizon} pursues this further. The sole unambiguous accuracy penalty occurs in advection, where rollout error climbs to $0.198$ against $0.125$ for the unconstrained FNO across all seeds, and only FNO manages to beat persistence ($0.170$). This we view as the inherent cost of imposing the constraint rather than a shortcoming: advection is perfectly reversible, so GENERIC-FNO is \emph{by design} unable to exploit its irreversible channel---there its dissipation diagnostic remains at $|r_{\mathrm{mech}}|\le0.035$ (\S\ref{sec:thermo})---thereby sacrificing the additional fitting freedom that a structurally unconstrained operator can devote to pure transport. The error gap and the near-zero $r_{\mathrm{mech}}$ express the same underlying phenomenon: when faced with a flow that generates no entropy, the model refuses to fabricate any. Turning to the DeepONet study (Table~\ref{tab:acc-don}), the construction transfers successfully to a non-spectral backbone: GENERIC-DeepONet surpasses persistence by $8\times$ on heat and $9\times$ on advection, while the unconstrained DeepONet baseline never even reaches persistence (Table~\ref{tab:acc-don}); notably, the advection penalty observed on the FNO backbone vanishes here, implying that this cost belongs to the spectral backbone rather than to the GENERIC framework. These DeepONet results are \emph{not} parameter-controlled, though: GENERIC-DeepONet includes the extra $E,S$ branch and is thus larger than the baseline ($175$K vs.\ $58$K), so by themselves they cannot isolate structure from capacity. Hence our parameter-controlled evidence comes in two forms: the 2D FNO comparison, where GENERIC-FNO takes Burgers (and heat in three of five seeds) at \emph{under half} the baseline's parameter count ($1.00$M vs.\ $2.10$M), and a budget-matched DeepONet ablation (Appendix~\ref{app:donmatch})---when trained at equal budget, both near $\approx\!58$K and near $\approx\!175$K, GENERIC-DeepONet achieves lower rollout error for every PDE at both budgets, and enlarging the baseline to $175$K fails to narrow the gap. The prominent DeepONet figures in Table~\ref{tab:acc-don} therefore stem from the construction itself rather than its greater parameter count. The DeepONet functionals also lack normalization, and GENERIC-DeepONet stands as the sole model in Tables~\ref{tab:acc-2d}--\ref{tab:acc-don} to beat persistence on heat by nearly an order of magnitude (\S\ref{sec:limitations}).

\paragraph{Where the constraint loses: the 1D FNO.}
The lower-capacity 1D FNO (Table~\ref{tab:acc-1d}; complete table in Appendix~\ref{app:1dfno}) tells a different story: GENERIC-FNO secures Burgers in every seed ($-40\%$ vs.\ FNO), matches FNO on heat but falls behind EP-FNO, and loses advection across all seeds. Two investigations pinpoint the reason. A capacity sweep (Appendix~\ref{app:capacity}; functional width varied over $17$K--$1.1$M parameters) reveals that the GENERIC/FNO error ratio behaves non-monotonically with width and plateaus above $70$K, ruling out capacity as the culprit. A normalization ablation (Appendix~\ref{app:underdiss}) instead attributes these losses largely to the instance normalization embedded in the functional networks: in 1D, where training proceeds without it, its removal cuts GENERIC-FNO's error by one to two orders of magnitude on each PDE (heat $0.081\!\to\!0.002$, advection $0.362\!\to\!0.005$). That same normalization proves necessary for training the 2D functionals (\S\ref{sec:limitations}), which explains why the model of record keeps it.

\begin{table}[ht]
\centering
\caption{\textbf{1D FNO, five-seed rollout accuracy} (relative $L^2$; one-step
errors in Appendix~\ref{app:1dfno}) with the persistence baseline.
GENERIC-FNO: $47$K parameters; baselines: $71$K. Bold = best learned model.
On this horizon every learned model is worse than persistence on heat and
Burgers, and only FNO beats it on advection; among the learned models GENERIC
wins Burgers in every seed, ties FNO on heat while losing to EP-FNO, and loses
advection in every seed.}
\label{tab:acc-1d}
\begin{tabular}{lcccc}
\toprule
PDE & Persistence & FNO & EP-FNO & GENERIC-FNO \\
\midrule
Heat      & $0.054$ & $0.082\pm0.012$ & $\mathbf{0.066\pm0.009}$ & $0.081\pm0.019$ \\
Advection & $0.201$ & $\mathbf{0.160\pm0.023}$ & $0.162\pm0.019$ & $0.362\pm0.069$ \\
Burgers   & $0.019$ & $0.070\pm0.033$ & $0.066\pm0.023$ & $\mathbf{0.042\pm0.015}$ \\
\bottomrule
\end{tabular}
\end{table}

\paragraph{Damped wave.}  
The damped wave equation, \(\partial_{tt} u = c^2\nabla^2 u - \gamma\,\partial_t u\), was also examined; it couples reversible propagation with irreversible damping, a regime for which GENERIC is specifically intended. In two dimensions, GENERIC-FNO produced a rollout error of \(0.105\pm0.008\), whereas FNO reached \(0.188\pm0.026\)—a \(44\%\) reduction across five seeds—while using only half the parameters. Yet against a persistence baseline of 0.103, this outcome indicates that GENERIC-FNO merely matches persistence, while FNO falls well short; among the reversible–irreversible problems tested, this is GENERIC-FNO's strongest relative advantage. The wave equation is omitted from the tables because it is second order in time, and since only the field \(u\) is observed—not its velocity \(\partial_t u\)—the one-step map is non-Markovian in the observed state. At the 2D resolution, spatial structure provides enough information about the velocity for all models to learn it, but in 1D with \(n_x=64\), the partial observability is so pronounced that every model, including the unconstrained FNO, degrades to a rollout error near 0.48 (Appendix~\ref{app:1dfno}). Consequently, the favorable, well-posed 2D result is reported in prose, and the wave equation is left out of the headline tables.

\subsection{Exact structure and zero-shot super-resolution}  
\label{sec:superres}  

Across all configurations above, the degeneracy identities from Section~\ref{sec:guarantees} are satisfied to machine precision, with residuals near \(10^{-13}\), and the discrete per-step energy residual of the explicit update is \(r_E \sim 10^{-6}\)–\(10^{-8}\). Since the operators work on a fixed band of Fourier modes, the projections are scale-invariant, and the variational derivatives are taken as \(L^2\) gradients (Appendix~\ref{app:setting}), these guarantees—not merely the accuracy—carry over to other resolutions. Table~\ref{tab:superres} and Figure~\ref{fig:resolution} show results from training at \(64\times64\) and testing zero-shot up to \(256\times256\) on band-limited data, where the maximum frequency is held constant so the continuous dynamics remain identical across grids. GENERIC-FNO's rollout accuracy stays flat or improves over the \(4\times\) range and is \(4\)–\(8\) times better than FNO at every grid, while the structural residual remains \(r_E \lesssim 10^{-5}\). The reversible diagnostic \(r_{\mathrm{mech}}\) also stays near zero for advection at all resolutions, confirming that the thermodynamic structure is preserved in the zero-shot setting. Advection's accuracy does decline at the coarsest grid for the reversible operator, but this stems from the explicit integrator—an effect we isolate and address in Section~\ref{sec:limitations}—rather than from any breakdown of the structure.

\begin{table}[ht]
\centering
\caption{\textbf{Zero-shot super-resolution.} Trained at $64$, evaluated zero-shot
to $256$ (rollout $L^2$). GENERIC-FNO accuracy is flat-to-improving and the
structural residual $r_E$ stays negligible; the model is genuinely
resolution-transferable in accuracy, with structural guarantees preserved.}
\label{tab:superres}
\begin{tabular}{llccccc}
\toprule
 & & \multicolumn{5}{c}{evaluation resolution} \\
\cmidrule(lr){3-7}
PDE & Model & 64 & 96 & 128 & 192 & 256 \\
\midrule
\multirow{2}{*}{Heat}
 & FNO         & 0.112 & 0.097 & 0.099 & 0.091 & 0.099 \\
 & GENERIC-FNO & $\mathbf{0.040}$ & $\mathbf{0.026}$ & $\mathbf{0.023}$ & $\mathbf{0.022}$ & $\mathbf{0.023}$ \\
\midrule
\multirow{2}{*}{Burgers}
 & FNO         & 0.107 & 0.145 & 0.131 & 0.164 & 0.124 \\
 & GENERIC-FNO & $\mathbf{0.033}$ & $\mathbf{0.023}$ & $\mathbf{0.021}$ & $\mathbf{0.019}$ & $\mathbf{0.015}$ \\
\midrule
\multicolumn{2}{l}{GENERIC-FNO $r_E$ (range)} & \multicolumn{5}{c}{$3\!\times\!10^{-7}\ \to\ 4\!\times\!10^{-5}$} \\
\bottomrule
\end{tabular}
\end{table}

\begin{figure}[ht!]
\centering
\includegraphics[width=\linewidth]{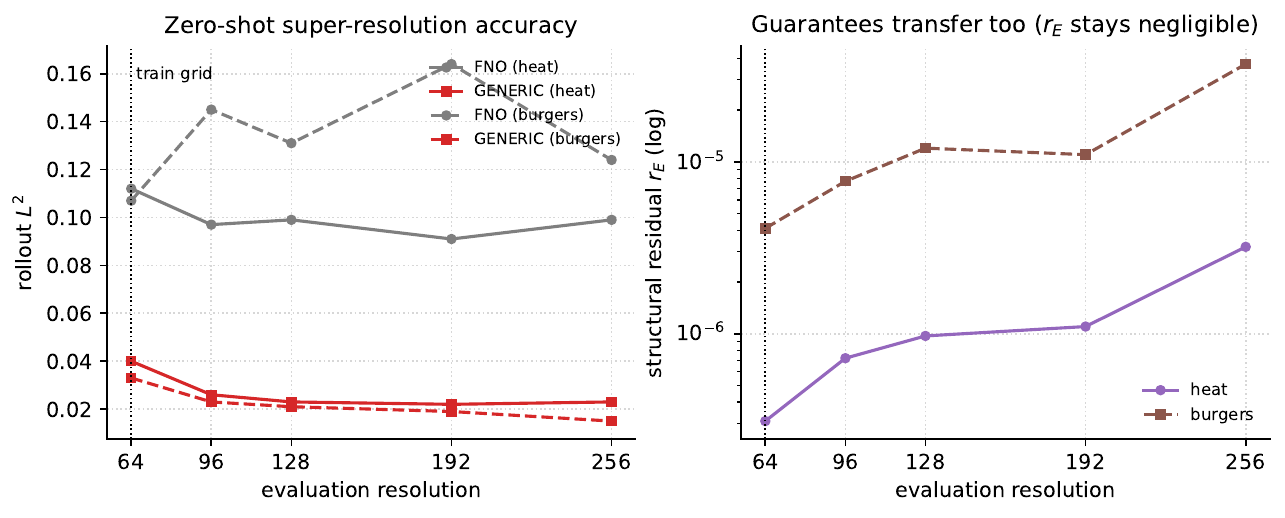}
\caption{\textbf{Zero-shot super-resolution (trained at $64$).}
\emph{Left:} rollout $L^2$ versus evaluation resolution for FNO and GENERIC-FNO on
heat and Burgers (flat-to-improving across the $4\times$ range). \emph{Right:}
the structural residual $r_E$ of GENERIC-FNO versus resolution (log scale),
verifying that the structural residual remains small under zero-shot transfer.
The dotted line marks the training grid.}
\label{fig:resolution}
\end{figure}

\subsection{Structural verification in three dimensions}
\label{sec:3d}

Because the construction in \S\ref{sec:operators} is dimension-agnostic at the level of code---one extra FFT axis carries the 2D model over to 3D---the structural identities can be tested without any ground truth. For a 3D GENERIC-FNO (functional width $16$, $6$ modes, $3$ layers; operator band $6$), the six identities of \S\ref{sec:guarantees} are reported in Table~\ref{tab:3d-identities}. They are evaluated in double precision on eight random band-limited fields at $32^3$, and on the same fields zero-padded at $64^3$, both at random initialization and on the weights of a small 3D heat model trained at $32^3$ ($48$ trajectories, $60$ epochs). Every identity holds to machine precision on either grid, trained or not. Larger absolute residuals appear for the trained model solely because its gradients are $\sim\!50\times$ larger; the normalized residual $r_E$ falls in the range $10^{-8}$--$10^{-10}$ throughout.

\begin{table}[ht]
\centering
\caption{\textbf{Structural identities in 3D} (double precision, maximum or
minimum over a batch of eight band-limited fields). The same fields are
evaluated at $32^3$ and zero-padded to $64^3$; ``trained'' is the 3D heat model
below.}
\label{tab:3d-identities}
\begin{tabular}{llcccc}
\toprule
 & & \multicolumn{2}{c}{random init} & \multicolumn{2}{c}{trained} \\
\cmidrule(lr){3-4}\cmidrule(lr){5-6}
Property & Identity & $32^3$ & $64^3$ & $32^3$ & $64^3$ \\
\midrule
Reversible skewness & $\langle \delta E, L\,\delta E\rangle=0$ & $8\!\times\!10^{-22}$ & $3\!\times\!10^{-21}$ & $3\!\times\!10^{-22}$ & $4\!\times\!10^{-21}$ \\
Energy degeneracy   & $\langle \delta E, M\,\delta S\rangle=0$ & $3\!\times\!10^{-14}$ & $3\!\times\!10^{-14}$ & $2\!\times\!10^{-11}$ & $2\!\times\!10^{-11}$ \\
Entropy degeneracy  & $\langle \delta S, L\,\delta E\rangle=0$ & $2\!\times\!10^{-14}$ & $2\!\times\!10^{-14}$ & $3\!\times\!10^{-16}$ & $3\!\times\!10^{-16}$ \\
Energy conservation & $\langle \delta E, \partial_t u\rangle=0$ & $3\!\times\!10^{-14}$ & $3\!\times\!10^{-14}$ & $2\!\times\!10^{-11}$ & $2\!\times\!10^{-11}$ \\
Entropy production  & $\langle \delta S, \partial_t u\rangle\ge0$ & $\ge+5\!\times\!10^{-6}$ & $\ge+4\!\times\!10^{-5}$ & $\ge+2\!\times\!10^{-3}$ & $\ge+2\!\times\!10^{-2}$ \\
Operator positivity & $\langle v, Mv\rangle\ge0$ & $\ge+275$ & $\ge+271$ & $\ge+275$ & $\ge+277$ \\
\midrule
Normalized $r_E$ (max) & & $2\!\times\!10^{-9}$ & $2\!\times\!10^{-10}$ & $1\!\times\!10^{-8}$ & $1\!\times\!10^{-9}$ \\
$\lVert\partial_t u\rVert$ (mean) & & $3.7\!\times\!10^{-3}$ & $1.1\!\times\!10^{-2}$ & $0.19$ & $0.55$ \\
\bottomrule
\end{tabular}
\end{table}

\paragraph{Proposition~\ref{prop:resinv} checked directly.}
Taking the variational derivatives as $L^2$ gradients---the autodiff gradient scaled by $N^3/|\Omega|$, App.~\ref{app:setting}---the $64^3$ update matches the prolonged $32^3$ update to $8.5\%$ of its norm at random initialization and to $7.1\%$ on the trained model, with the update-function norms matching to $0.3\%$ and $0.7\%$. Aliasing of the pointwise nonlinearities in $E_\theta,S_\phi$ accounts for this residual (\S\ref{sec:limitations}). If the scaling is dropped, the update function contracts by a factor $0.125=1/2^3$ with each refinement and the step ceases to commute; this is precisely why the scaling enters the proposition as a hypothesis. Under either convention, the identities of Table~\ref{tab:3d-identities} remain valid.

\paragraph{Trained 3D heat model.}
Measured against the exact spectral heat solution at $32^3$ ($\nu=0.02$, $\Delta t=0.005$, ten-step rollout on $16$ held-out trajectories), the 3D GENERIC-FNO achieves a one-step error of $0.0023$ and a rollout error of $0.0128$, whereas a 3D FNO of comparable size reaches $0.0189$ and $0.107$; at every rollout step, $r_E\le 9\!\times\!10^{-8}$. As in 2D (\S\ref{sec:accuracy}), the persistence baseline ($0.0125$) is competitive with the constrained model over this short horizon and far ahead of the unconstrained one. The aim of this run is to verify the identities on trained weights in 3D rather than to add an accuracy claim.

\subsection{Gauge-Invariant Dissipation}
\label{sec:thermo}

Table~\ref{tab:gauge} (and Figure~\ref{fig:dissipation} in Appendix~\ref{app:figures}) report the gauge-invariant dissipation rate \(r_{\mathrm{mech}}\) against the ground-truth rate \(\Pi^\star\) across all three backbones. The load-bearing result is a pass/fail test, not a correlation: on the reversible advection system, the model neither dissipates nor injects physical energy beyond \(|r_{\mathrm{mech}}|\le 0.035\) in any seed of any backbone. This sits an order of magnitude below the least-dissipative irreversible system and a factor of \(10\)–\(40\) below heat, while \(r_{\mathrm{mech}}\) is large and positive for heat in every backbone and every seed. The sign of the small advection values is not resolved (2D FNO: \(-0.007\pm0.014\)), and need not be: the construction conserves the \emph{learned} energy \(E_\theta\), not \(Q\), so on a reversible flow \(r_{\mathrm{mech}}\) measures the gauge mismatch between \(E_\theta\) and \(Q\) and can carry either sign. Reversible data admits no positive entropy-production representation, so the near-vanishing of \(|r_{\mathrm{mech}}|\) is the one thermodynamic statement that cannot be a gauge artifact.

A second, weaker observation concerns ordering. Heat has the largest \(r_{\mathrm{mech}}\) in every backbone and every seed. Between Burgers and advection, the picture depends on the backbone: DeepONet separates them cleanly (\(+0.26\) vs.\ \(+0.01\), five of five seeds), the 2D FNO does so in the mean and in four of five seeds with overlapping spreads (\(+0.008\pm0.010\) vs.\ \(-0.007\pm0.014\)), and the 1D FNO does not resolve them (\(+0.017\pm0.024\) vs.\ \(+0.018\pm0.026\); three of five seeds). Stated plainly, a rank correlation over three points per backbone carries little power on its own, and on two backbones the two least-dissipative systems lie within the noise floor of the diagnostic. Pooled over the nine (backbone, PDE) means, the Spearman correlation with \(\Pi^\star\) is \(\rho=0.73\) (exact permutation \(p=0.016\); Appendix~\ref{app:figures}, Fig.~\ref{fig:rmech_ordering}), and even that is a rank result over a small set of canonical systems. The reversible test is decisive; the separation of heat from the rest is robust; the Burgers-versus-advection ordering is resolved only where the model's dissipation on Burgers is well above that floor. We do not claim more than this.

We compare \(r_{\mathrm{mech}}\) and \(\Pi^\star\) \emph{ordinally}, not by magnitude: \(r_{\mathrm{mech}}\in[-1,1]\) is a normalized alignment (a cosine between \(u\) and \(\partial_t u\)), whereas \(\Pi^\star\) is a fractional one-step change of \(Q\), so the two carry different normalizations and are not expected to agree in absolute value (e.g., heat on the 2D FNO: \(r_{\mathrm{mech}}=0.31\) vs.\ \(\Pi^\star=0.031\)); only the sign and the ranking are claimed. A related caveat: a positive \(r_{\mathrm{mech}}\) records a decay of the fixed energy \(Q=\tfrac12\|u\|^2\)—it would flag \emph{any} norm-shrinking update. This is meaningful here because for the scalar PDEs we study, \(Q\) \emph{is} the physical mechanical energy that the true dynamics conserve (advection) or dissipate (heat, Burgers), so its decay is genuine physical dissipation rather than an artifact; on a system where \(Q\) is not the relevant invariant, \(r_{\mathrm{mech}}\) would be redefined with the appropriate energy. The damped wave is excluded from this table for exactly that reason: observed through \(u\) alone, the decay of \(Q\) on the observed field is exchange with the unobserved velocity, not dissipation, and its \(\Pi^\star\) does not have a stable sign across seeds. Because both quantities reference only the fixed energy, the conclusions are independent of the gauge-dependent learned decomposition.

An illustrative single-rollout view of the learned \(E_\theta\) and \(S_\phi\) appears in Appendix~\ref{app:figures} (Fig.~\ref{fig:interpretability}); being gauge-dependent, it is shown only as interpretive color.

\begin{table}[ht]
\centering
\caption{\textbf{Gauge-invariant dissipation} (five-seed mean~$\pm$~std).
$r_{\mathrm{mech}}\!\approx\!0$ identifies the reversible system and heat is
the most dissipative in every backbone and seed; the Burgers-versus-advection
ordering is resolved on DeepONet and the 2D FNO (five and four of five seeds)
and not on the 1D FNO. Pooled over the nine means, Spearman $\rho=0.73$
($p=0.016$). The two columns use different normalizations---$r_{\mathrm{mech}}$
is a cosine rate, $\Pi^\star$ a fractional one-step energy change---so they are
compared ordinally, not by magnitude.}
\label{tab:gauge}
\begin{tabular}{llcc}
\toprule
Backbone & PDE & $r_{\mathrm{mech}}$ & $\Pi^\star$ (truth) \\
\midrule
\multirow{3}{*}{2D FNO}
 & Heat      & $+0.314\pm0.078$ & $+3.1\!\times\!10^{-2}$ \\
 & Advection & $-0.007\pm0.014$ & $+1.5\!\times\!10^{-7}$ \\
 & Burgers   & $+0.008\pm0.010$ & $+1.2\!\times\!10^{-3}$ \\
\midrule
\multirow{3}{*}{1D FNO}
 & Heat      & $+0.295\pm0.087$ & $+1.6\!\times\!10^{-2}$ \\
 & Advection & $+0.018\pm0.026$ & $+9.0\!\times\!10^{-8}$ \\
 & Burgers   & $+0.017\pm0.024$ & $+1.5\!\times\!10^{-3}$ \\
\midrule
\multirow{3}{*}{1D DeepONet}
 & Heat      & $+0.811\pm0.025$ & $+1.6\!\times\!10^{-2}$ \\
 & Advection & $+0.009\pm0.010$ & $+7.9\!\times\!10^{-8}$ \\
 & Burgers   & $+0.259\pm0.039$ & $+1.3\!\times\!10^{-3}$ \\
\bottomrule
\end{tabular}
\end{table}


\subsection{Long-Horizon Stability}
\label{sec:longhorizon}

Stability far from the training data is what motivates the thermodynamic
structure, and we assess it by rolling each model out for $200$ steps---over
$13\times$ the $15$-step training horizon---on held-out trajectories. For these
runs all five models are trained under one lighter protocol per PDE than that
of Section~\ref{sec:accuracy}, so short-horizon errors exceed those in
Table~\ref{tab:acc-2d}, although the across-model comparison stays controlled,
with a single run per cell. Beyond the paper's own baselines we include two
unconstrained models chosen to test whether a more expressive architecture
would be stable without a prior: FNO-L, an FNO sized so that its measured
training-step time under this protocol matches GENERIC-FNO's ($37.8$M
parameters), nd a residual U-Net with periodic padding~\citep{ronneberger2015unet}, the standard convolutional baseline for PDE surrogates~\citep{takamoto2022pdebench,gupta2022multiscale}. We report relative
rollout error at the training horizon and at $200$ steps
(Table~\ref{tab:longhorizon}), together with the physical energy
$Q=\tfrac12\|u\|^2$ traced along the rollout (Fig.~\ref{fig:long_horizon}).
This is also the horizon at which persistence breaks down in energy: the true
heat field sheds $86\%$ of its $Q$ over those $200$ steps, so the identity
map---competitive in the ten-step tables of Section~\ref{sec:accuracy}---holds
an energy far from the truth here, and we report it as a row.

\begin{table}[ht]
\centering
\caption{\textbf{Long-horizon rollout stability} (single run per cell;
training horizon $15$ steps; error normalized by the initial reference norm,
so a prediction that decays to zero scores exactly $1$). Two unconstrained
baselines are added to the paper's: FNO-L, an FNO sized so its measured
training-step time under this protocol matches GENERIC-FNO's ($385$ vs.\
$303$\,ms), and a residual U-Net. On heat and Burgers every unconstrained
model diverges by $200$ steps, and the larger ones diverge faster; GENERIC-FNO
stays bounded and holds $Q$ near its initial value---which is also why its
error is close to persistence there (see text). On advection, EP-FNO and the
U-Net reach $L^2\!\approx\!1$ by collapsing the field to zero ($Q/Q_0\to0.02$);
GENERIC-FNO is the only model that retains the energy.}
\label{tab:longhorizon}
\begin{tabular}{llccccc}
\toprule
PDE & Model & Params & step (ms) & $L^2$@15 & $L^2$@200 & $Q/Q_0$@200 \\
\midrule
\multirow{6}{*}{Heat (truth $Q/Q_0$: 0.14)}
 & Persistence & --     & --  & 0.252 & 0.852 & 1.00 \\
 & FNO         & 2.10M  & 61  & 1.102 & 8.809 & 40.9 \\
 & EP-FNO      & 2.10M  & 62  & 1.001 & 6.845 & 24.8 \\
 & FNO-L       & 37.8M  & 385 & 4.602 & 61.0  & $2\!\times\!10^{3}$ \\
 & U-Net       & 1.93M  & 124 & 2.378 & $>10^{15}$ & $>10^{30}$ \\
 & GENERIC-FNO & 1.00M  & 303 & 0.257 & \textbf{0.813} & 0.89 \\
\midrule
\multirow{6}{*}{Advection (truth: 1.00)}
 & Persistence & --     & --  & 0.472 & 1.451 & 1.00 \\
 & FNO         & 2.10M  & 63  & 1.285 & 10.37 & 84.6 \\
 & EP-FNO      & 2.10M  & 63  & 1.015 & 1.012 & 0.02 \\
 & FNO-L       & 37.8M  & 388 & 2.794 & 32.37 & 853 \\
 & U-Net       & 1.93M  & 125 & 1.008 & 1.012 & 0.02 \\
 & GENERIC-FNO & 1.00M  & 302 & 0.420 & 1.654 & \textbf{1.14} \\
\midrule
\multirow{6}{*}{Burgers (truth: 0.61)}
 & Persistence & --     & --  & 0.040 & 0.347 & 1.00 \\
 & FNO         & 2.10M  & 63  & 1.139 & 4.283 & 8.8 \\
 & EP-FNO      & 2.10M  & 63  & 1.525 & 11.94 & 67.2 \\
 & FNO-L       & 37.8M  & 387 & 1.482 & 10.54 & 51.9 \\
 & U-Net       & 1.93M  & 125 & 1.508 & 2.178 & 1.89 \\
 & GENERIC-FNO & 1.00M  & 304 & 0.054 & \textbf{0.364} & 1.02 \\
\bottomrule
\end{tabular}
\end{table}
On the dissipative and mixed problems every unconstrained model exits the
admissible set, and the larger ones exit it faster. After $200$ steps the
FNO's physical energy has climbed to roughly $41\times$ its starting value on
heat and $9\times$ on Burgers; the compute-matched FNO-L reaches about
$2000\times$ and $50\times$; the U-Net blows up numerically on heat and
roughly doubles its energy on Burgers; and the energy penalty delays but does
not prevent the drift ($25\times$ and $67\times$). GENERIC-FNO is the only
model whose energy remains of order one throughout ($0.89$ and $1.02$), and
its rollout error is bounded accordingly ($0.81$ and $0.36$). We state plainly
what that boundedness is: those errors sit close to the persistence values
($0.85$ and $0.35$), because GENERIC-FNO holds $Q$ near its initial value while
the true field relaxes to $0.14$ and $0.61$ of it. The constraint prevents the
spurious energy injection that destroys every unconstrained model here; it
does not by itself supply the physical dissipation, a point we return to
below. Extra unconstrained capacity does not help---FNO-L is worse than the
FNO it enlarges on all three PDEs.

For reversible advection, where the true energy stays exactly constant, the
rollout error is not the discriminating quantity, and the table shows why.
Under our normalization a prediction that has decayed to zero scores exactly
$1$, and that is what EP-FNO and the U-Net do ($Q/Q_0\to0.02$): they attain
the lowest advection errors by annihilating the field. FNO and FNO-L instead
inflate it ($85\times$ and $850\times$). GENERIC-FNO is the only model that
retains the energy, holding $Q$ to within $14\%$ over $200$ steps at the cost
of transport phase error ($1.65$, comparable to persistence at $1.45$). Its
deviation takes the form of a slow, steady upward drift
(Fig.~\ref{fig:long_horizon}), which reflects the integrated effect of the
slightly negative mean $r_{\mathrm{mech}}$ reported for advection in
Table~\ref{tab:gauge}: although the learned energy is conserved exactly at
every step, its gauge mismatch with $Q$ accumulates. What the structural prior
governs is the learned energy---neither the phase nor $Q$ itself.

Long-time heat most clearly reveals what the construction guarantees and what
it does not. GENERIC-FNO under-dissipates: rather than relaxing toward the true
$0.14$, its physical energy levels off near $0.9$, and that plateau is the
whole of its residual $0.81$ error. Staying bounded is not the same as being
correct: the construction preserves the learned energy $E_\theta$ to machine
precision at each step, yet nothing in it restricts the physical energy $Q$,
and in a 1D replication of the experiment $Q$ not only plateaus but rises above
its initial value (Appendix~\ref{app:underdiss}). The underlying cause is
architectural, not structural: instance normalization within the $E,S$
backbones renders the learned functionals nearly scale-invariant in $u$, so
the update magnitude cannot diminish as the field relaxes, preventing the
model from converging (\S\ref{sec:limitations}, Appendix~\ref{app:stability}).
Removing that normalization in an ablation---with identical parameter count
and protocol---yields exact tracking of the heat decay across all $200$ steps
and no energy-increasing updates (Appendix~\ref{app:underdiss}); nevertheless,
the same normalization is essential for training the 2D functionals at all, so
we keep the reviewed model as the model of record and leave a normalization
that satisfies both requirements to future work.

\begin{figure}[ht!]
\centering
\includegraphics[width=\linewidth]{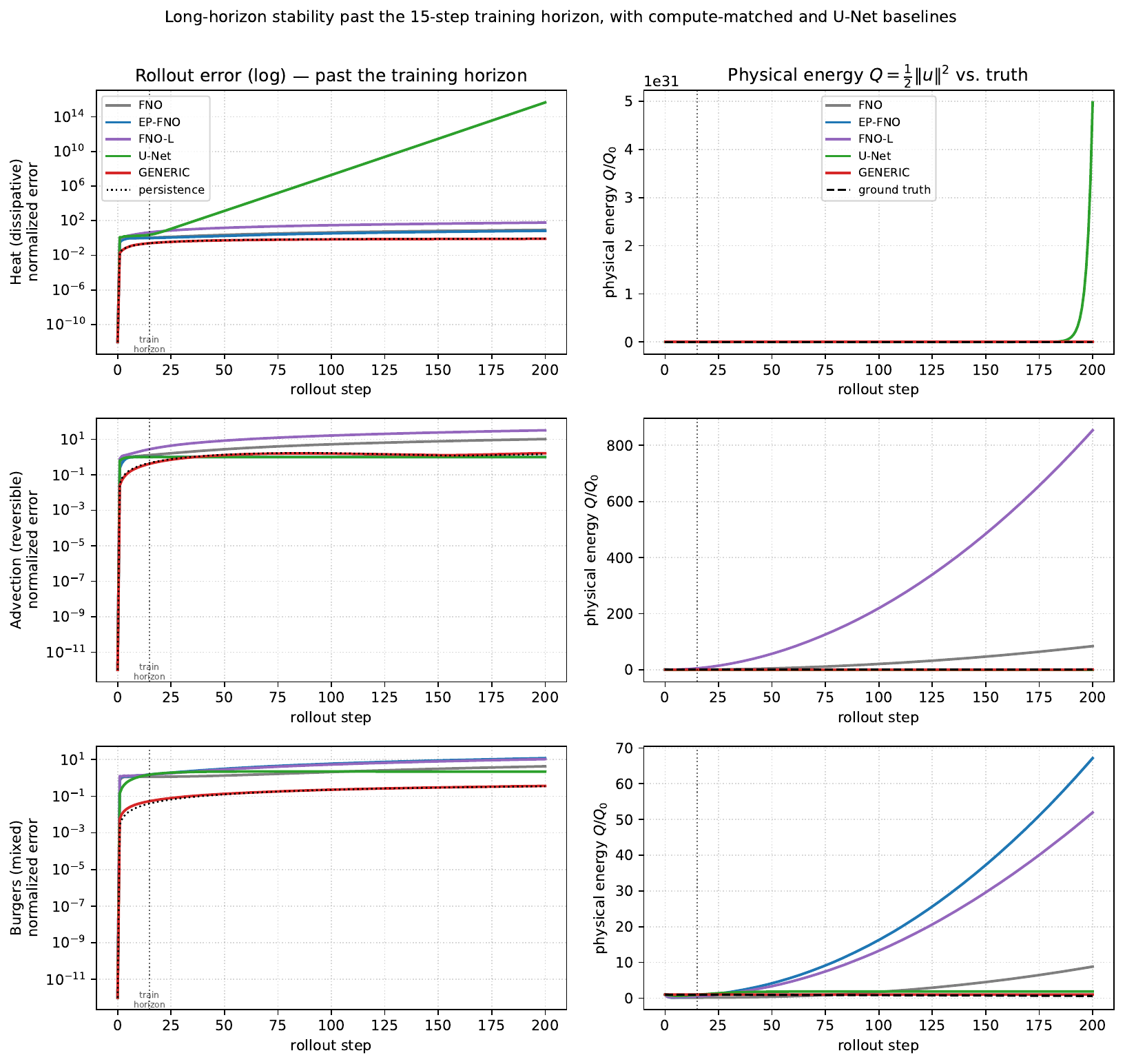}
\caption{\textbf{Long-horizon rollouts} ($200$ steps; dotted vertical line marks the
$15$-step training horizon; rollout-error vertical axes are log-scaled,
while physical-energy axes are linear). Five models plus persistence (dotted). \emph{Left:} rollout error vs.\
step---every unconstrained model diverges on heat and Burgers, the compute-matched
FNO-L and the U-Net faster than the small FNO, while GENERIC-FNO stays bounded near the
persistence level. \emph{Right:} physical energy $Q=\tfrac12\|u\|^2$ relative to its
initial value, against the ground truth (dashed). FNO's energy grows to
$\approx\!41\times$ on heat and $\approx\!9\times$ on Burgers, FNO-L to
$\approx\!2000\times$ and $\approx\!50\times$, and the U-Net blows up numerically on heat;
on advection EP-FNO and the U-Net collapse the field to zero ($Q/Q_0\to0.02$) while FNO
and FNO-L inflate it. GENERIC-FNO is the only model whose energy stays of order one on
all three, holding the constant true energy on advection to within $\approx\!14\%$ (a
slow upward drift, the integrated gauge mismatch between $E_\theta$ and $Q$). On heat it
under-dissipates---its energy plateaus near $0.9$ above the true decay to $0.14$---so it
is bounded but not correct at long times (see text and Appendix~\ref{app:underdiss}).}
\label{fig:long_horizon}
\end{figure}

\section{Limitations and Scope}
\label{sec:limitations}

We state explicitly where the method is constrained, both to guide its use and because several of these limitations illuminate the construction itself.

\paragraph{The Jacobi identity is not enforced.} The construction guarantees exact skew-adjointness of \(L\), positive semi-definiteness of \(M\), and both degeneracy conditions (Section \ref{sec:operators}); it does not, however, ensure that \(L\) is a verified Poisson operator. The base multiplier \(D_L\) is constant and therefore Poisson by itself, but the projection makes \(L\) state-dependent, and we neither impose nor verify the Jacobi identity for the projected bracket. The first and second laws proved in Appendix \ref{app:proofs} rely only on skewness, positivity, and degeneracy, so no result in this paper depends on Jacobi. What is lost is the geometric interpretation of the reversible channel as a Hamiltonian flow on a Poisson manifold. This same limitation applies to finite-dimensional learned-bracket methods with state-dependent \(L\) (Section \ref{sec:related}); resolving it in function space remains an open problem.

\paragraph{Short-horizon benchmarks are near-persistence.} With the time steps used in Tables \ref{tab:acc-2d} and \ref{tab:acc-1d}, the fields change little over the ten-step evaluation horizon, making the persistence map \(u_{t+1}=u_t\) competitive or superior. In 2D, persistence is within noise of the model of record on heat and outperforms every learned model on Burgers; in 1D, it beats every learned FNO model on both heat and Burgers; and the unconstrained DeepONet baseline is indistinguishable from persistence on all PDEs. The ten-step tables thus measure how well a model avoids corrupting a slowly varying field rather than how well it advances one, and we do not base the accuracy claim on them. The comparisons that actually differentiate the models are the \(200\)-step rollouts in Section \ref{sec:longhorizon}, where the true heat field loses \(86\%\) of its energy, persistence fails in energy, every unconstrained model diverges or collapses (including a compute-matched FNO and a U-Net), and GENERIC-FNO remains bounded, together with the DeepONet study (Table \ref{tab:acc-don}), where the constrained model outperforms persistence on heat by nearly an order of magnitude. Benchmarks with larger time steps, or with evaluation horizons chosen so that persistence is far from the truth, would make the short-horizon tables more informative; we have not regenerated the data to do so.

\paragraph{Reversible transport and the explicit integrator.} The clearest accuracy cost of the constraint appears on pure linear transport with the FNO backbone, where GENERIC-FNO is less accurate than an unconstrained FNO in every seed (Table \ref{tab:acc-2d}), and where FNO is the only model that beats persistence. A sharper manifestation occurs in the super-resolution study: at the coarsest training grid, advection rollout error is anomalously high and worsens as the grid coarsens (Table \ref{tab:euler-rk4}). We traced this issue to the explicit integrator rather than to the structure. Explicit Euler applied to the skew operator yields a per-step amplitude growth \(\propto\sqrt{1+(\omega\Delta t)^2}\); the pointwise nonlinearities in the \(E,S\) networks inject high-frequency harmonics that alias on a coarse grid, raising the effective \(\omega\), and because the reversible regime has \(M\!\approx\!0\), there is no dissipation to damp the aliased modes. The structural diagnostics remain clean throughout (\(r_{\mathrm{mech}}\approx 0\), degeneracy \(\sim\!10^{-13}\)), confirming an integration effect rather than a breakdown of degeneracy. A fourth-order Runge–Kutta integrator with anti-aliased nonlinearities substantially mitigates the blow-up but raises the per-step energy residual from \(\sim\!10^{-7}\) to \(\sim\!10^{-3}\) (Table \ref{tab:euler-rk4}); we retain explicit Euler as the model of record because it gives the smaller measured per-step energy residual, and we note that discrete-gradient or projection-based integrators offer a concrete path to both accuracy and the exact discrete guarantee.

\begin{table}[t]
\centering
\caption{\textbf{Advection under explicit Euler vs.\ RK4 + anti-aliasing}
(rollout $L^2$, trained at $64$). RK4 roughly halves the coarse-grid blow-up but
raises the per-step energy residual $r_E$ from $\sim\!10^{-7}$ to
$\sim\!10^{-3}$. We retain Euler because it gives the smaller measured energy residual
in this comparison.}
\label{tab:euler-rk4}
\begin{tabular}{lccccc}
\toprule
 & FNO & \multicolumn{2}{c}{GENERIC (Euler)} & \multicolumn{2}{c}{GENERIC (RK4 + AA)} \\
\cmidrule(lr){2-2}\cmidrule(lr){3-4}\cmidrule(lr){5-6}
res & $L^2$ & $L^2$ & $r_E$ & $L^2$ & $r_E$ \\
\midrule
64 (train) & 0.174 & 0.606 & $2\!\times\!10^{-8}$ & 0.341 & $2.9\!\times\!10^{-3}$ \\
128        & 0.134 & 0.154 & $9\!\times\!10^{-8}$ & 0.153 & $1.2\!\times\!10^{-3}$ \\
256        & 0.117 & 0.161 & $1\!\times\!10^{-7}$ & 0.159 & $8\!\times\!10^{-5}$ \\
\bottomrule
\end{tabular}
\end{table}

\paragraph{Operator expressivity limits the reversible channel.} The reversible part of the architecture is built from deliberately simple operators—diagonal Fourier multipliers restricted to a low-frequency band and sandwiched behind rank-one projections—and this simplicity is what bounds how expressive that part can be on pure transport. It is precisely here that we locate the 2D advection loss. To test whether that loss stems from insufficient capacity, we sweep the width of the 1D backbone (Appendix~\ref{app:capacity}), varying the functional networks from \(17\)K to \(1.1\)M parameters under the paper's 1D protocol. Across all three PDEs, the GENERIC/FNO rollout-error ratio is non-monotone in width and essentially flat once the networks exceed roughly \(70\)K parameters, so enlarging \(E_\theta\) or \(S_\phi\) does not close any gap; the losses are not a capacity effect. Importantly, the exact degeneracy does not hinge on the diagonal form of the multiplier: for any bounded \(D\), the projection sandwich preserves skew-adjointness and positive semi-definiteness (Lemma~\ref{lem:sandwich}, Proposition~\ref{prop:degeneracy}). Consequently, richer structure-preserving parameterizations—multi-rank or block-banded multipliers—would affect only expressivity, never exactness, and we defer them to future work.

\paragraph{Normalization in the functional networks.} Training the \(E,S\) backbones requires instance normalization. Without it, the variational derivatives are \(\sim\!200\times\) smaller at initialization (Appendix~\ref{app:stability}, probe (e)), and at the 2D resolution studied, the model never departs from the identity map (Appendix~\ref{app:underdiss}). But normalization carries a price. Because each sample is normalized by its own statistics, the learned functionals become nearly scale-invariant in \(u\); as a result, \(\lVert\delta E/\delta u\rVert\) grows like \(1/\lVert u\rVert\) near the zero field, and the update cannot shrink as the field relaxes. A heat gradient flow requires an entropy that is quadratic in \(u\), which a scale-invariant functional cannot represent across amplitudes. This is the diagnosed cause of the long-time heat under-dissipation reported in \S\ref{sec:longhorizon} and accounts for most of the 1D-FNO accuracy deficit. In 1D, where the networks are small enough to train without normalization, removing it cuts GENERIC-FNO's rollout error by one to two orders of magnitude on every PDE—heat \(0.081\!\to\!0.002\), advection \(0.362\!\to\!0.005\), Burgers \(0.042\!\to\!0.011\) (five seeds, Appendix~\ref{app:underdiss})—and eliminates the physical-energy monotonicity violations. The DeepONet backbone, whose functionals carry no normalization, shows the same signature (Table~\ref{tab:acc-don}). Because the same change prevents training at 2D scale, it is not a drop-in fix; we keep the reviewed model as the model of record and leave a normalization that trains at scale without imposing scale-invariance in \(u\) to future work.

\paragraph{Partial observation.} GENERIC-FNO models a first-order, Markovian map in the observed field. Systems that are not Markovian in the observed state—most directly, second-order dynamics such as the wave equation when only the displacement is observed—are only approximately representable. This is benign in 2D but defeats every model, ours and the unconstrained baseline alike, at the coarse 1D grid (\S\ref{sec:accuracy}). Lifting to an augmented state (e.g., \((u, \partial_t u)\)) is a natural extension we do not pursue here.

\paragraph{Computational cost.} Exact variational derivatives require two reverse-mode passes per step, and training backpropagates through them. At \(128^2\), inference is \(\approx\!4.4\times\), a training step \(\approx\!10\times\), and peak memory \(\approx\!2.3\times\) a plain FNO, at half the parameters (Appendix~\ref{app:repro}). For DeepONet, the second functional branch also makes GENERIC-DeepONet larger than its baseline (175K vs.\ 58K); the parameter-matched ablation of Appendix~\ref{app:donmatch} controls for this.

\paragraph{Gauge-dependence of interpretation.} As established in \S\ref{sec:gauge}, the learned \((E, S, L, M)\) decomposition is not unique. We therefore make no claim about the absolute learned energy or entropy, nor about per-channel quantities such as \(\rho_M\), which vary with the gauge. Our thermodynamic claims are restricted to gauge-invariant quantities: the exact structural identities and the diagnostic \(r_{\mathrm{mech}}\). For these, we claim the reversibility test, the separation of heat from the other systems in every backbone, and the Burgers-versus-advection ordering only where it is resolved (\S\ref{sec:thermo})—not a calibrated dissipation rate. We also do not characterize the gauge group itself: Appendix~\ref{app:gauge} proves invariance of our diagnostics under gauge equivalence, and a pointwise argument there shows the freedom is not confined to scalar rescaling (the seed-to-seed swings of \(\rho_M\) in Figure~\ref{fig:gauge_variance} are the empirical signature), but we are not aware of a classification of all tuples generating a given flow. The one principled selection of a representative we know of uses information a macroscopic trajectory does not carry—the large-deviation structure of an underlying particle system, in the cases where it is available \citep{mielke2014gradient,duong2013generic}—so we treat the decomposition as a gauge rather than attempt to learn \emph{the} entropy.

\paragraph{Scope of the formalism, and of the study.} Two scope statements must be kept separate. First, the scope of GENERIC itself: the formalism describes closed systems with a well-defined energy and entropy—conservative, dissipative, or coupled. Non-conservative, externally forced, or open systems lie outside it, and hence outside any method built on it, including ours. Applied to such data, the architecture would still impose an \((E, S, L, M)\) split and report machine-precision degeneracy residuals, because those residuals are algebraic properties of the parameterization, not a test of whether the data possess metriplectic structure (see Broader Impact). Second, the scope of this study: scalar fields on periodic domains for a small set of canonical PDEs spanning the reversible–irreversible spectrum. These are a deliberate controlled testbed—each has a known ground-truth dissipation rate, which is what makes the falsifiable reversibility test of \S\ref{sec:thermo} and the exact super-resolution check of \S\ref{sec:superres} possible—not the target application. Vector and multi-field systems, non-periodic boundaries, and the complex-fluid and coarse-graining problems that motivate GENERIC are the intended use; the structural identities are verified in three dimensions (\S\ref{sec:3d}), but 3D accuracy at application scale is not studied here. Demonstrating the construction on those problems is the main direction we leave open.
\section{Conclusion}
\label{sec:conclusion}
We present GENERIC-FNO, a neural operator that, to the best of our knowledge,
is the first to impose both GENERIC degeneracy conditions of non-equilibrium
thermodynamics directly in function space. Energy and entropy functionals are
learned as neural operators, while the reversible and irreversible operators
are parameterized as diagonal Fourier multipliers bounded by rank-one
projections. This parameterization makes the degeneracy conditions---and hence
conservation of the learned energy and production of the learned
entropy---hold exactly and by construction, requiring neither penalty terms,
nor projection of the predicted update, nor any residual freedom. Because the
Jacobi identity for the state-dependent reversible bracket is not imposed, the
structure qualifies as metriplectic-degenerate rather than a complete
Poisson--metriplectic pair.

Regardless of initialization, spatial dimension, or resolution, the structural
identities are satisfied to machine precision: they generalize zero-shot across
a $4\times$ super-resolution range in 2D and hold in three dimensions, at
random initialization and after training. Our analysis further uncovers a gauge
freedom in the learned $(E,S,L,M)$ decomposition, allowing us to separate
gauge-invariant results from those that depend on the gauge choice. We also
introduce a dissipation diagnostic that is both falsifiable and
gauge-invariant: it detects reversibility and pinpoints the most dissipative
system within each backbone, all without consulting the learned functionals.
Empirically, results vary by backbone across three operator architectures and
four PDEs, and we document this variation in full. The prior improves accuracy
on dissipative and mixed problems for the 2D FNO and DeepONet
backbones---achieving this in 2D with half the parameters, albeit at roughly
$4$--$10\times$ the computational cost of a plain FNO---and keeps the
$200$-step rollouts bounded where unconstrained baselines diverge, while it
under performs on pure linear transport and, for the small 1D backbone, on
advection, offering no benefit on heat. A capacity sweep indicates these losses
do not stem from limited capacity, and an ablation traces them to the operator
parameterization in 2D transport and to a normalization artifact in the 1D
functional networks.

The method's limitations each suggest a concrete avenue for future work:
developing richer structure-preserving operator parameterizations beyond
diagonal multipliers, with exactness already ensured by the projection
sandwich; designing a normalization for the functional networks that permits
training at scale without enforcing scale-invariance, which an ablation
indicates would eliminate the diagnosed near-equilibrium under-dissipation;
adopting structure-preserving higher-order integrators that maintain exact
discrete degeneracy; using augmented observed states for second-order
dynamics; and extending to vector- and multi-field systems on non-periodic
domains. Still, the most promising direction is the one that originally
motivated this formalism: applying exactly thermodynamically consistent neural
operators to closure and complex-fluid problems, where remaining on the
physical manifold far from the training data determines whether a surrogate
proves stable or useless.
\section*{Broader Impact}

GENERIC-FNO guarantees exact conservation of a learned energy and non-negative production of a learned entropy. This guarantee is structural, of the model's own functionals, and does not certify physical correctness. A particular failure mode follows directly from the design and deserves explicit acknowledgment: when applied to a non-conservative, externally forced, or open system—outside GENERIC's scope of closed conservative–dissipative dynamics—the architecture will nonetheless impose an (E, S, L, M) decomposition and report machine-precision degeneracy residuals, since those residuals are an algebraic consequence of the parameterization and hold for any data whatsoever. Exact degeneracy, therefore, should never be interpreted as evidence that the learned dynamics are physically valid. This paper itself demonstrates that gap within its domain: in the long-horizon heat experiment of Section~\ref{sec:longhorizon}, the learned energy remains conserved to machine precision at every step, even as the physical energy fails to relax to its true limit. We particularly warn against using the machine-precision guarantee as a mark of correctness in safety-critical simulation or engineering decisions, such as those in structural, aerospace, or process-safety contexts. The intended benefit is improved stability and physical plausibility for learned surrogates operating within GENERIC's domain; deploying the method on real systems demands independent verification that the target system falls within that class and that the learned dynamics align with the governing equations, boundary conditions, and operating regime. The guarantees we establish are necessary structural conditions, not replacements for such validation.
\bibliography{tmlr}
\bibliographystyle{tmlr}

\appendix
\section{Mathematical Background and Proofs}
\label{app:proofs}

This appendix makes the claims of \S\ref{sec:background}--\S\ref{sec:gauge}
precise and self-contained. Throughout, all operators are understood pointwise in
the state $u$: $L=L(u)$ and $M=M(u)$ depend on $u$ through the gradients $\delta
E/\delta u$ and $\delta S/\delta u$, and every statement below is asserted at a
fixed but arbitrary $u$.

\subsection{A gentle introduction to GENERIC for non-specialists}
\label{app:primer}

This subsection provides intuition for readers new to the GENERIC formalism; it is self-contained, assuming only vector calculus and linear algebra. Although nothing here is needed for the proofs that follow, it should make the main text more accessible.

\paragraph{The problem GENERIC solves.} Classical dynamics offers two complementary pictures, each capturing only part of physics. In Hamiltonian mechanics, motion is reversible and conserves energy: a skew-symmetric operator \(L\) generates dynamics from an energy \(E\) via \(\dot u = L\,\nabla E\), and since \(L\) is skew, \(\dot E = \langle \nabla E, L\,\nabla E\rangle = 0\)—energy is exactly conserved and the dynamics are time-reversible, as in planetary orbits or ideal waves. Gradient flows represent the opposite extreme: an energy is monotonically dissipated through \(\dot u = -\,\nabla E\), yielding \(\dot E = -\lVert\nabla E\rVert^2 \le 0\), as when heat spreads or friction brings a system to rest. Real systems, such as a viscous fluid or a cooling gas, exhibit both behaviors simultaneously: they transport and oscillate while irreversibly generating heat and entropy. GENERIC~\citep{grmela1997dynamics,ottinger2005beyond} provides a framework that carries both aspects at once, keeping them from interfering.

\paragraph{The two-generator equation.} GENERIC employs two potentials—an energy \(E\) and an entropy \(S\)—along with two operators: a skew-adjoint \(L\) for the reversible, Hamiltonian part, and a positive semi-definite \(M\) for the irreversible, dissipative part:
\[
  \partial_t u \;=\; \underbrace{L\,\tfrac{\delta E}{\delta u}}_{\text{reversible}}
                \;+\; \underbrace{M\,\tfrac{\delta S}{\delta u}}_{\text{irreversible}}.
\]
(For a field \(u(x)\), the gradient \(\nabla E\) becomes the variational derivative \(\delta E/\delta u\); \S\ref{app:setting} makes this precise, but it behaves like a gradient.) The first term rotates the state while conserving \(E\); the second slides it downhill in a way that conserves nothing but always increases \(S\).

\paragraph{Why the degeneracy conditions matter.} For this to give clean thermodynamics—energy exactly conserved (first law), entropy never decreasing (second law)—the two channels must not contaminate each other. That is exactly what the degeneracy conditions \(L\,\delta S/\delta u = 0\) and \(M\,\delta E/\delta u = 0\) enforce: the reversible operator produces no entropy, and the irreversible operator changes no energy. A two-line computation (\S\ref{app:laws}) then gives \(\dot E = 0\) and \(\dot S \ge 0\). The technical heart of this paper is making these two conditions hold exactly, at every state and resolution, rather than approximately through a penalty—which is what the projection sandwich of Eq.~\eqref{eq:construction} achieves.

\paragraph{The analogy for a vision reader} If it helps: \(L\) plays the role of a learned skew (rotation-like) operator and \(M\) a learned symmetric PSD (smoothing-like) operator, while \(E\) and \(S\) are two learned scalar "energy" heads. The projections \(I-P_S\) and \(I-P_E\) are rank-one orthogonal projections (as in removing a component along a fixed direction), applied so that each operator is blind to the other's gradient. Everything else in the paper is about carrying this construction faithfully into function space, where it must be independent of grid resolution.

\subsection{Function-space setting}
\label{app:setting}
Let \(\Omega = \mathbb{T}^d\) be the periodic torus and \(\mathcal{H} = L^2(\Omega;\mathbb{R})\) the real Hilbert space with inner product \(\langle f,g\rangle = \int_\Omega f(x)g(x)\,dx\) and norm \(\lVert f\rVert = \langle f,f\rangle^{1/2}\). For a Fréchet-differentiable functional \(F:\mathcal{H}\to\mathbb{R}\), the variational derivative \(\delta F/\delta u\in\mathcal{H}\) is the Riesz representative of the differential, \(\langle \delta F/\delta u,\,\eta\rangle = \lim_{\epsilon\to0}\epsilon^{-1}\big(F[u+\epsilon\eta]-F[u]\big)\) for all \(\eta\in\mathcal{H}\).

On an \(N\)-point grid, \(\mathcal{H}\) is replaced by \(\mathbb{R}^N\) with \(\langle f,g\rangle = \Delta x\sum_i f_i g_i\). Our functionals are mean-pooled, \(F[u] = h(\tfrac1N\sum_i g(u)_i)\), so the Euclidean gradient returned by automatic differentiation, \(\nabla_u F\), equals \(\tfrac{\Delta x}{N}\) times the discrete Riesz representative of \(\delta F/\delta u\). Every property used below (orthogonal projection, skew-adjointness, positive semi-definiteness, the degeneracy identities) is invariant to a positive scalar multiple of \(\delta E/\delta u\) and \(\delta S/\delta u\), so this constant is immaterial for every identity below and we write
$\delta F/\delta u$ for the autodiff gradient without further comment. It is
not immaterial for the magnitude of the update across grids: resolution
transfer requires the $L^2$ scaling $N^d/|\Omega|$ (Proposition~\ref{prop:resinv}).

\subsection{Rank-one projections}
\label{app:proj}
For $w\in\mathcal{H}\setminus\{0\}$ define
$P_w v = \frac{\langle w,v\rangle}{\langle w,w\rangle}\,w$.

\begin{lemma}\label{lem:proj}
$P_w$ is the orthogonal projection onto $\mathrm{span}\{w\}$: it is self-adjoint
($P_w^\ast = P_w$) and idempotent ($P_w^2 = P_w$), $P_w w = w$, and
$(I-P_w)w = 0$. Consequently $I-P_w$ is the self-adjoint, idempotent orthogonal
projection onto $w^\perp$.
\end{lemma}
\begin{proof}
For any $u,v$, $\langle P_w u, v\rangle = \frac{\langle w,u\rangle}{\langle
w,w\rangle}\langle w,v\rangle = \langle u, P_w v\rangle$, so $P_w^\ast=P_w$.
$P_w^2 v = \frac{\langle w,v\rangle}{\langle w,w\rangle}P_w w =
\frac{\langle w,v\rangle}{\langle w,w\rangle}w = P_w v$ since $P_w w = w$.
Finally $(I-P_w)w = w - w = 0$. Self-adjointness and idempotence of $I-P_w$ are
immediate, and its range is $w^\perp$ because $P_w$ has range
$\mathrm{span}\{w\}$.
\end{proof}

\subsection{The base Fourier multipliers}
\label{app:multipliers}
Let $\widehat{(\cdot)}$ denote the Fourier transform on $\mathbb{T}^d$, so that
$\langle f,g\rangle = \sum_k \overline{\widehat f_k}\,\widehat g_k$ (Parseval),
and recall that $f$ is real iff $\widehat f_{-k} = \overline{\widehat f_k}$. Let
$D_L$ and $D_M$ be the diagonal multipliers with symbols
$\widehat{D_L}(k) = \mathrm{i}\,a(k)$ and $\widehat{D_M}(k) = |b(k)|^2$, where
$a(k)\in\mathbb{R}$ is odd ($a(-k)=-a(k)$) and $|b(k)|^2\ge 0$ is even.

\begin{proposition}\label{prop:DL-DM}
$D_L$ maps real fields to real fields and is skew-adjoint,
$\langle f, D_L g\rangle = -\langle D_L f, g\rangle$. $D_M$ maps real fields to
real fields and is self-adjoint and positive semi-definite,
$\langle f, D_M f\rangle \ge 0$.
\end{proposition}
\begin{proof}
Reality of $D_L f$ follows from $\widehat{D_L f}_{-k} = \mathrm{i}\,a(-k)\widehat
f_{-k} = -\mathrm{i}\,a(k)\overline{\widehat f_k} = \overline{\mathrm{i}\,a(k)\widehat
f_k} = \overline{\widehat{D_L f}_k}$, using that $a$ is odd. By Parseval,
\[
\langle f, D_L g\rangle = \sum_k \overline{\widehat f_k}\,\mathrm{i}\,a(k)\widehat g_k,
\qquad
\langle D_L f, g\rangle = \sum_k \overline{\mathrm{i}\,a(k)\widehat f_k}\,\widehat g_k
= -\sum_k \mathrm{i}\,a(k)\overline{\widehat f_k}\,\widehat g_k,
\]
so $\langle f, D_L g\rangle = -\langle D_L f, g\rangle$. For $D_M$, evenness of
$|b(k)|^2$ gives reality by the same computation, self-adjointness because the
symbol is real, and
$\langle f, D_M f\rangle = \sum_k |b(k)|^2\,|\widehat f_k|^2 \ge 0$.
\end{proof}

In the implementation the inverse real FFT enforces the Hermitian symmetry that
makes $a(k)$ effectively odd and $|b(k)|^2$ even, so these properties hold for
\emph{any} values of the learnable parameters.

\subsection{Structure preservation under the projection sandwich}
\label{app:sandwich}

\begin{lemma}\label{lem:sandwich}
Let $P$ be a self-adjoint projection and $D$ a bounded operator. Then:
(i) if $D$ is skew-adjoint, $(I-P)D(I-P)$ is skew-adjoint; (ii) if $D$ is
self-adjoint and positive semi-definite, $(I-P)D(I-P)$ is self-adjoint and
positive semi-definite.
\end{lemma}
\begin{proof}
Write $Q=I-P$, which is self-adjoint ($Q^\ast=Q$) by Lemma~\ref{lem:proj}. Then
$(QDQ)^\ast = Q^\ast D^\ast Q^\ast = QD^\ast Q$. (i) If $D^\ast=-D$ then
$(QDQ)^\ast = -QDQ$. (ii) If $D^\ast=D$ then $(QDQ)^\ast = QDQ$, and for any $v$,
$\langle v, QDQ v\rangle = \langle Qv, D\,Qv\rangle \ge 0$ since $D\succeq0$.
\end{proof}

\begin{corollary}\label{cor:LM-structure}
With $L=(I-P_S)D_L(I-P_S)$ and $M=(I-P_E)D_M(I-P_E)$, the operator $L$ is
skew-adjoint and $M$ is self-adjoint positive semi-definite.
\end{corollary}
\begin{proof}
Combine Proposition~\ref{prop:DL-DM} with Lemma~\ref{lem:sandwich}, taking
$P=P_S$ for $L$ and $P=P_E$ for $M$.
\end{proof}

\subsection{Exact degeneracy}
\label{app:degeneracy}

\begin{proposition}\label{prop:degeneracy}
$L\,\dfrac{\delta S}{\delta u} = 0$ and $M\,\dfrac{\delta E}{\delta u} = 0$ hold
identically, for any parameters and any state $u$.
\end{proposition}
\begin{proof}
By Lemma~\ref{lem:proj}, $(I-P_S)\,\delta S/\delta u = 0$. Hence
$L\,\delta S/\delta u = (I-P_S)D_L\big[(I-P_S)\,\delta S/\delta u\big]
= (I-P_S)D_L\,0 = 0$. Symmetrically $(I-P_E)\,\delta E/\delta u = 0$ gives
$M\,\delta E/\delta u = 0$.
\end{proof}

\subsection{The first and second laws}
\label{app:laws}

\begin{theorem}\label{thm:laws}
Let $\partial_t u = L\,\delta E/\delta u + M\,\delta S/\delta u$ with $L,M$ as
above. Then the semi-discrete (continuous-time) dynamics conserve energy and
produce entropy:
\[
\frac{dE}{dt} = 0,
\qquad
\frac{dS}{dt} = \Big\langle \tfrac{\delta S}{\delta u},\, M\,\tfrac{\delta S}{\delta u}\Big\rangle \ge 0 .
\]
\end{theorem}
\begin{proof}
By the chain rule $dE/dt = \langle \delta E/\delta u,\,\partial_t u\rangle$.
Expanding,
\[
\frac{dE}{dt}
= \Big\langle \tfrac{\delta E}{\delta u}, L\,\tfrac{\delta E}{\delta u}\Big\rangle
+ \Big\langle \tfrac{\delta E}{\delta u}, M\,\tfrac{\delta S}{\delta u}\Big\rangle .
\]
The first term vanishes because $L$ is skew-adjoint (Corollary~\ref{cor:LM-structure}):
$\langle v, Lv\rangle = -\langle Lv, v\rangle = -\langle v, Lv\rangle$ forces
$\langle v,Lv\rangle=0$. The second vanishes because $M\,\delta E/\delta u=0$
(Proposition~\ref{prop:degeneracy}) and $M$ is self-adjoint, so
$\langle \delta E/\delta u, M\,\delta S/\delta u\rangle = \langle M\,\delta
E/\delta u,\,\delta S/\delta u\rangle = 0$. Likewise
\[
\frac{dS}{dt}
= \Big\langle \tfrac{\delta S}{\delta u}, L\,\tfrac{\delta E}{\delta u}\Big\rangle
+ \Big\langle \tfrac{\delta S}{\delta u}, M\,\tfrac{\delta S}{\delta u}\Big\rangle ;
\]
the first term equals $-\langle L\,\delta S/\delta u,\,\delta E/\delta u\rangle =
0$ by skew-adjointness and Proposition~\ref{prop:degeneracy}, and the second is
$\ge 0$ by positive semi-definiteness of $M$ (Corollary~\ref{cor:LM-structure}).
\end{proof}

\begin{proposition}[Discrete energy drift]\label{prop:drift}
Let $u_{t+1} = u_t + \Delta t\, f(u_t)$ with $f = L\,\delta E/\delta u + M\,\delta
S/\delta u$ the explicit Euler update, and suppose $E$ is $C^2$ with Hessian
bounded by $\beta$ on the trajectory. Then
$\big|E[u_{t+1}] - E[u_t]\big| \le \tfrac12\beta\,\Delta t^2\,\lVert f(u_t)\rVert^2 = \mathcal{O}(\Delta t^2)$.
\end{proposition}
\begin{proof}
By Taylor's theorem, $E[u_t+\Delta t f] = E[u_t] + \Delta t\langle \delta
E/\delta u, f\rangle + R$ with $|R|\le \tfrac12\beta\Delta t^2\lVert f\rVert^2$.
The first-order term $\langle \delta E/\delta u, f\rangle = dE/dt$ vanishes by
Theorem~\ref{thm:laws}, leaving $|E[u_{t+1}]-E[u_t]| = |R|$.
\end{proof}

Thus energy conservation is \emph{exact in continuous time} (the first-order
term is identically zero, verified numerically to $\sim\!10^{-13}$ in
\S\ref{sec:guarantees}); the explicit update incurs only the standard
$\mathcal{O}(\Delta t^2)$ integrator error, which a higher-order integrator
reduces at the cost discussed in \S\ref{sec:limitations}.

\subsection{Gauge freedom and gauge-invariance of the diagnostics}
\label{app:gauge}

\begin{definition}\label{def:gauge}
Two tuples $(E,S,L,M)$ and $(E',S',L',M')$ are \emph{gauge-equivalent} on a
domain $\mathcal{U}\subseteq\mathcal{H}$ if they generate the same flow,
\[
L(u)\,\frac{\delta E}{\delta u} + M(u)\,\frac{\delta S}{\delta u}
=
L'(u)\,\frac{\delta E'}{\delta u} + M'(u)\,\frac{\delta S'}{\delta u}
=: f(u) \qquad \text{for all } u\in\mathcal{U}.
\]
\end{definition}

\begin{proposition}\label{prop:gauge-inv}
The dissipation diagnostic
$r_{\mathrm{mech}}(u) = -\langle u, f(u)\rangle/(\lVert u\rVert\,\lVert f(u)\rVert)$
and the ground-truth rate $\Pi^\star$ are gauge-invariant: gauge-equivalent
tuples yield identical values. Moreover, if the flow conserves the quadratic
energy $Q[u]=\tfrac12\lVert u\rVert^2$ (a reversible flow), then
$r_{\mathrm{mech}}(u)=0$ for all $u$.
\end{proposition}
\begin{proof}
$r_{\mathrm{mech}}$ is a function of the pair $(u, f(u))$ alone; by
Definition~\ref{def:gauge} gauge-equivalent tuples share the same $f$, hence the
same $r_{\mathrm{mech}}$. $\Pi^\star$ is computed from the reference trajectory
and does not reference $(E,S,L,M)$ at all. For the reversible case,
$\tfrac{d}{dt}Q[u] = \langle u, \partial_t u\rangle = \langle u, f(u)\rangle$;
if $Q$ is conserved this is zero, so the numerator of $r_{\mathrm{mech}}$
vanishes.
\end{proof}

\begin{remark}
By contrast, the dissipative fraction $\rho_M = \lVert M\,\delta S/\delta
u\rVert/(\lVert L\,\delta E/\delta u\rVert+\lVert M\,\delta S/\delta u\rVert)$ is
\emph{not} gauge-invariant: it depends on the individual generators $L\,\delta
E/\delta u$ and $M\,\delta S/\delta u$, which gauge-equivalent tuples may
apportion differently while preserving their sum $f$. This is why we report
$\rho_M$ only as an appendix mechanism diagnostic and base all thermodynamic
claims on $r_{\mathrm{mech}}$.
\end{remark}

\begin{remark}[Extent of the gauge freedom]
We do not characterize the gauge group, and the results above do not require
it. We note only that it is not confined to the trivial rescalings
$(E,L)\mapsto(\lambda E,L/\lambda)$, $(S,M)\mapsto(\mu S,M/\mu)$. Pointwise,
the structure constrains the split $f = L\,\delta E/\delta u + M\,\delta
S/\delta u$ by four conditions---the reversible part is orthogonal to
$\delta E/\delta u$ (skewness) and to $\delta S/\delta u$ (degeneracy of
$L$), the dissipative part is orthogonal to $\delta E/\delta u$ (degeneracy
of $M$), and $\langle\delta S/\delta u, M\,\delta S/\delta u\rangle\ge0$---which
for a given $f$ fix the single scalar $\langle\delta S/\delta u, f\rangle$ and
otherwise leave the split free on a codimension-two subspace at each state.
Any functional conserved by $f$ can therefore serve as $E$, and any functional
monotone along $f$ as $S$, with admissible operators available for each
choice away from the degenerate configurations
$\delta E/\delta u \parallel \delta S/\delta u$ or
$\langle\delta S/\delta u, f\rangle=0$. Empirically, the trivial rescalings
leave $\rho_M$ unchanged, so the seed-to-seed swings of $\rho_M$ in
Figure~\ref{fig:gauge_variance} ($0.001$ to $1.000$ for the same 2D Burgers
flow at unchanged accuracy and $r_{\mathrm{mech}}$) are an existence
demonstration that the realized freedom exceeds rescaling. We are not aware
of a characterization of the full gauge group for general flows. That the
same macroscopic evolution admits different dissipative structures is well
documented, and in the cases treated by \citet{mielke2014gradient} and
\citet{duong2013generic} the large-deviation principle of an underlying
particle system yields a particular one, though the former also note the
absence of a canonical selection in general; a model trained only on
macroscopic trajectories has no such information, which is why we treat the
decomposition as a gauge rather than attempt to learn \emph{the} entropy.
\end{remark}

\subsection{Resolution and dimension independence}
\label{app:resolution}
Let $\Pi_N:\mathcal{H}\to V_N$ be the orthogonal projection onto the span of the
first $N$ Fourier modes (the grid of resolution $N$), and let the operator band
$m$ satisfy $2m \le N$. Write $\Phi$ for one GENERIC-FNO step, with the
variational derivatives taken as $L^2$ gradients, i.e.\ the autodiff gradient
scaled by $N^d/|\Omega|$ (Appendix~\ref{app:setting}).

\begin{proposition}\label{prop:resinv}
For any band-limited field $u$ with $\widehat u_k = 0$ for $|k|>K$, and any two
resolutions $N,N'$ with $2m \le N \le N'$ and $K \le m$, the step commutes with
prolongation up to the aliasing of the pointwise nonlinearities in the
functional networks: $\Phi_{N'}(\iota_{N\to N'}u) = \iota_{N\to N'}\Phi_N(u)$
whenever the functionals are band-limited, where $\iota_{N\to N'}$ is
zero-padding in Fourier (band-limited interpolation). The identities of
Proposition~\ref{prop:degeneracy} and Theorem~\ref{thm:laws} hold at every
resolution unconditionally.
\end{proposition}

\begin{proof}[Proof sketch]
The multipliers $D_L,D_M$ act diagonally on modes $|k|\le m$ and as zero
outside, identically at any $N\ge 2m$; zero-padding adds only modes $|k|>m$,
which both $\Phi_N$ and $\Phi_{N'}$ leave untouched. The spectral-convolution
and mean-pooling layers of the functional networks are defined per-mode and by
the resolution-independent spatial average, so for band-limited functionals
$\delta E/\delta u$ and $\delta S/\delta u$ agree as $L^2$ functions across
$N,N'$; this is where the $L^2$ scaling enters, since the raw autodiff gradient
carries the grid factor $|\Omega|/N^d$ of Appendix~\ref{app:setting} and would
shrink the update by that factor per refinement. The projections $P_E,P_S$ are
ratios of inner products and are invariant to the grid measure under either
convention, which is why the identities hold unconditionally. Hence every
band-limited component of $\Phi$ commutes with $\iota_{N\to N'}$; the pointwise
nonlinearities are the only components that are not band-limited, and their
aliased content is the residual.
\end{proof}

Proposition~\ref{prop:resinv} is the formal content of the zero-shot
super-resolution result (\S\ref{sec:superres}): under the $L^2$ scaling, a model
trained at one resolution applies unchanged at any finer one, and the
structural guarantees of Theorem~\ref{thm:laws} and
Proposition~\ref{prop:degeneracy}, holding pointwise at every $u$, transfer with
it under any scaling. Both statements are checked directly in
\S\ref{sec:superres}: in 2D (Table~\ref{tab:superres}) and in 3D
(Table~\ref{tab:3d-identities}), where with the $L^2$ scaling the $64^3$ update
agrees with the prolonged $32^3$ update to within $8.5\%$ (random
initialization) and $7.1\%$ (trained) of its norm---the aliasing residual---and
without it the update shrinks by exactly $1/2^3$ per refinement while every
identity is unchanged. The aliasing of the pointwise nonlinearities is also the
source of the coarse-grid effect analyzed in \S\ref{sec:limitations}.

\section{1D Fourier Neural Operator Results}
\label{app:1dfno}

Table~\ref{tab:1d-full} reports the complete five-seed 1D FNO study referenced in \S\ref{sec:accuracy}, including one-step errors and the damped-wave row. The 1D FNO is the smallest-capacity backbone we examine (GENERIC-FNO uses $47$K parameters versus $71$K for the baselines) and the one on which the structural prior helps least: GENERIC wins Burgers in every seed, ties FNO on heat while losing to EP-FNO, and loses advection in every seed. Appendix~\ref{app:capacity} rules out capacity as the cause, and Appendix~\ref{app:underdiss} traces most of the deficit to the instance normalization in the functional networks. The damped-wave row substantiates the partial-observation claim of \S\ref{sec:accuracy}: at $n_x=64$, with only the displacement $u$ observed and not its velocity, every model---including the unconstrained FNO---collapses to a rollout error of approximately $0.48$, so wave is excluded from the main-text tables.

\begin{table}[ht]
\centering
\caption{\textbf{1D FNO, five-seed accuracy} (relative $L^2$, mean~$\pm$~std).
GENERIC-FNO: $47$K parameters; FNO and EP-FNO: $71$K. Bold = best per PDE on
rollout. Wave entries are the partial-observation collapse (all models
$\approx 0.48$).}
\label{tab:1d-full}
\begin{tabular}{llcc}
\toprule
PDE & Model & $L^2$ (1-step) & $L^2$ (rollout) \\
\midrule
\multirow{4}{*}{Heat}
 & FNO         & $0.016\pm0.003$ & $0.082\pm0.012$ \\
 & EP-FNO      & $0.013\pm0.002$ & $\mathbf{0.066\pm0.009}$ \\
 & GENERIC-FNO & $0.022\pm0.015$ & $0.081\pm0.019$ \\
  & Persistence & $0.010$ & $0.054$ \\      
\midrule
\multirow{4}{*}{Advection}
 & FNO         & $0.031\pm0.002$ & $\mathbf{0.160\pm0.023}$ \\
 & EP-FNO      & $0.032\pm0.003$ & $0.162\pm0.019$ \\
 & GENERIC-FNO & $0.090\pm0.045$ & $0.362\pm0.069$ \\
  & Persistence & $0.037$ & $0.201$ \\      
\midrule
\multirow{4}{*}{Burgers}
 & FNO         & $0.012\pm0.004$ & $0.070\pm0.033$ \\
 & EP-FNO      & $0.012\pm0.005$ & $0.066\pm0.023$ \\
 & GENERIC-FNO & $0.008\pm0.003$ & $\mathbf{0.042\pm0.015}$ \\
  & Persistence & $0.004$ & $0.019$ \\      
\midrule
\multirow{4}{*}{Wave}
 & FNO         & $0.083\pm0.010$ & $0.490\pm0.050$ \\
 & EP-FNO      & $0.082\pm0.010$ & $0.490\pm0.041$ \\
 & GENERIC-FNO & $0.073\pm0.012$ & $\mathbf{0.479\pm0.045}$ \\
 & Persistence & $0.059$ & $0.439$ \\      
\bottomrule
\end{tabular}
\end{table}

\subsection{1D-FNO capacity sweep}
\label{app:capacity}

To investigate whether the 1D-FNO losses on heat and advection (Appendix~\ref{app:1dfno}) reflect a backbone-capacity limitation, we scale the backbone over four capacity levels and test whether the gap to the unconstrained FNO diminishes. In step, we increase the FNO/EP-FNO width across $\{16,32,64,128\}$ (the paper uses $32$) and the GENERIC functional width \texttt{width\_func} across $\{12,24,48,96\}$ (the paper uses $24$), keeping both model families proportionally scaled and spanning $17$K--$1.1$M parameters. With all other settings and the \texttt{train\_model}/\texttt{evaluate\_model} code identical to the 1D protocol of Appendix~\ref{app:1dfno} ($150$ epochs, $200$ train / $50$ test trajectories, single seed), we report the $10$-step autoregressive rollout $L^2$ error.

\begin{table}[ht]
\centering
\caption{\textbf{1D-FNO capacity sweep} (rollout $L^2$, single seed). Params
are FNO\,/\,GENERIC. The rightmost column is the GENERIC/FNO error ratio; the
paper's configuration is L1. The ratio is \emph{non-monotone} in width on every
PDE and essentially flat from L1 upward, so enlarging the functional networks
does not close the gap---the losses are not a capacity effect.}
\label{tab:capacity}
\begin{tabular}{llccccc}
\toprule
PDE & Params (FNO/GEN) & FNO & EP-FNO & GENERIC-FNO & GEN/FNO \\
\midrule
\multirow{4}{*}{Heat}
 & L0\ \ $17$K\,/\,$12$K   & 0.112 & 0.111 & \textbf{0.074} & 0.66 \\
 & L1\ \ $71$K\,/\,$47$K   & \textbf{0.054} & 0.054 & 0.084 & 1.56 \\
 & L2\ \ $283$K\,/\,$185$K & 0.063 & \textbf{0.063} & 0.079 & 1.25 \\
 & L3\ \ $1.1$M\,/\,$739$K & 0.073 & \textbf{0.072} & 0.091 & 1.26 \\
\midrule
\multirow{4}{*}{Advection}
 & L0\ \ $17$K\,/\,$12$K   & 0.140 & \textbf{0.139} & 0.212 & 1.52 \\
 & L1\ \ $71$K\,/\,$47$K   & \textbf{0.118} & 0.118 & 0.276 & 2.34 \\
 & L2\ \ $283$K\,/\,$185$K & 0.116 & \textbf{0.113} & 0.261 & 2.25 \\
 & L3\ \ $1.1$M\,/\,$739$K & \textbf{0.118} & 0.118 & 0.243 & 2.07 \\
\midrule
\multirow{4}{*}{Burgers}
 & L0\ \ $17$K\,/\,$12$K   & 0.256 & 0.255 & \textbf{0.031} & 0.12 \\
 & L1\ \ $71$K\,/\,$47$K   & \textbf{0.060} & 0.061 & 0.090 & 1.50 \\
 & L2\ \ $283$K\,/\,$185$K & \textbf{0.054} & 0.055 & 0.078 & 1.46 \\
 & L3\ \ $1.1$M\,/\,$739$K & 0.061 & \textbf{0.059} & 0.089 & 1.46 \\
\bottomrule
\end{tabular}
\end{table}

Across the three test cases, the GENERIC/FNO ratio traces a non-monotone path before settling: on heat it moves through \(0.66\), \(1.56\), \(1.25\), and \(1.26\); on advection through \(1.52\), \(2.34\), \(2.25\), and \(2.07\); and on Burgers through \(0.12\), \(1.50\), \(1.46\), and \(1.46\). In each instance the sequence becomes flat once the backbone reaches L1. Raising the number of functional parameters by a factor of sixteen between L1 and L3 shifts the ratio by no more than a few percent. Consequently, the 1D losses do not stem from the expressiveness of \(E_\theta\) or \(S_\phi\); instead, Appendix~\ref{app:underdiss} traces most of the shortfall to an artifact of the instance normalization employed in those networks. Two cells at level L0 deviate sharply, but this reflects the smallest FNO width rather than any systemic feature—an under-parameterized FNO handles Burgers poorly (yielding \(0.256\)), whereas the same small FNO performs comparatively well on heat. The pattern worth attending to is therefore the L1–L3 stretch, which remains essentially unchanged. Since these results come from single seeds, and given the seed variability of the 1D backbone documented in Appendix~\ref{app:1dfno}, the conclusion should be drawn from the level-to-level trajectory, not from any isolated cell.

\subsection{Numerical stability of the rank-one projections as the learned gradients vanish}
\label{app:stability}

Near a critical point of \(E_\theta\) or \(S_\phi\), where \(\langle\delta E/\delta u,\delta E/\delta u\rangle\) or \(\langle\delta S/\delta u,\delta S/\delta u\rangle\) tends to zero, the denominators in the projections approach zero as noted in \S\ref{sec:operators}. The projection itself is computed as \(P_w v = \langle w,v\rangle\,w/(\langle w,w\rangle+\varepsilon)\), with \(\varepsilon = 10^{-12}\) in single precision. To characterize its limiting behavior, we run four probes: (a) the operator in isolation, sweeping \(\lVert w\rVert\) across fourteen orders of magnitude; (b) within the model, letting the trained functional gradients shrink via \(\tau \to 0\); (c) within the model, at states genuinely close to equilibrium, \(u = A\,u_0\) as \(A \to 0\); and (d) a 2000-step heat rollout progressing toward equilibrium. A fifth probe (e) pinpoints the origin of a scale effect uncovered by (c). Every in-model probe employs a GENERIC-FNO trained on 1D heat.

\paragraph{(a) Operator level.}
For an arbitrary \(w\), Cauchy–Schwarz ensures \(\lVert(I-P_w)v\rVert \le \lVert v\rVert\), meaning the projection acts as a contraction and cannot magnify the input; it also guarantees that \((I-P_w)v\) approaches \(v\) smoothly as \(w\) goes to zero. Table~\ref{tab:stab-a} verifies both features: the ratio never exceeds \(1\), and the correction term \(\lVert P_w v\rVert\) decays along with \(\lVert w\rVert\). When \(\lVert w\rVert^2\) falls well below \(\varepsilon\), the projection effectively becomes the identity, which causes no harm—the direction it would eliminate is itself nearly zero, so the degeneracy identities hold in a trivial manner.

\begin{table}[ht]
\centering
\caption{\textbf{(a) Operator-level stability.} $v,w$ unit-normalized;
$\lVert w\rVert$ swept. Ratio $\lVert(I-P_w)v\rVert/\lVert v\rVert\le 1$
everywhere; correction $\to 0$; all finite.}
\label{tab:stab-a}
\begin{tabular}{lccc}
\toprule
$\lVert w\rVert$ & $\lVert(I-P_w)v\rVert/\lVert v\rVert$ & $\lVert P_w v\rVert$ & $\lVert(I-P_w)w\rVert/\lVert w\rVert$ \\
\midrule
$10^{0}$   & 0.986272 & $1.65\!\times\!10^{-1}$ & $0.00\!\times\!10^{0}$ \\
$10^{-2}$  & 0.986272 & $1.65\!\times\!10^{-1}$ & $0.00\!\times\!10^{0}$ \\
$10^{-4}$  & 0.986272 & $1.65\!\times\!10^{-1}$ & $1.00\!\times\!10^{-4}$ \\
$10^{-5}$  & 0.986274 & $1.63\!\times\!10^{-1}$ & $9.90\!\times\!10^{-3}$ \\
$10^{-6}$  & 0.989722 & $8.26\!\times\!10^{-2}$ & $5.00\!\times\!10^{-1}$ \\
$10^{-7}$  & 0.999731 & $1.63\!\times\!10^{-3}$ & $9.90\!\times\!10^{-1}$ \\
$10^{-10}$ & 1.000000 & $2.33\!\times\!10^{-10}$ & $1.00\!\times\!10^{0}$ \\
$10^{-12}$ & 1.000000 & $0.00\!\times\!10^{0}$ & $1.00\!\times\!10^{0}$ \\
$10^{-14}$ & 1.000000 & $0.00\!\times\!10^{0}$ & $1.00\!\times\!10^{0}$ \\
\bottomrule
\end{tabular}
\end{table}

\paragraph{(b) In-model, gradients scaled by $\tau$.}

The exact construction, when supplied with $\tau\,\delta E/\delta u$ and $\tau\,\delta S/\delta u$, drives $\langle w, w\rangle$ to zero. Linearity in $\tau$ is exact throughout---the $\lVert\text{update}\rVert/\tau$ column remains constant---and stability holds down to $\tau = 10^{-9}$, where absolute degeneracy residuals stay at or below $10^{-14}$ (Table \ref{tab:stab-b}). Only once $\lVert\tau\,\delta E\rVert^2$ drops below $\varepsilon$ does the dimensionless ratio $r_E$ climb to $\mathcal{O}(0.1)$; yet in that regime the update magnitude is under $5\times10^{-6}$, and the absolute energy drift $\lvert\langle\delta E/\delta u, \partial_t u\rangle\rvert$ sits near $10^{-14}$. The floor therefore substitutes boundedness for a vanishing normalized ratio exactly where the update has already become negligible, and lowering that floor further requires either double precision or a smaller $\varepsilon$.

\begin{table}[ht]
\centering
\caption{\textbf{(b) In-model, gradients scaled by $\tau$.} Update
$\propto\tau$; degeneracy residuals $|\langle\delta E,M\delta S\rangle|$,
$|\langle\delta S,L\delta E\rangle|$ at machine level; entropy production
$\min\,\mathrm{d}S/\mathrm{d}t\ge 0$; all finite. Base
$\lVert\delta E\rVert=1.2\!\times\!10^{-2}$,
$\lVert\delta S\rVert=5.8\!\times\!10^{-1}$.}
\label{tab:stab-b}
\begin{tabular}{lccccc}
\toprule
$\tau$ & $\lVert\text{upd}\rVert$ & $\lVert\text{upd}\rVert/\tau$ & $r_E$ & $|\langle\delta E,M\delta S\rangle|$ & $\min\,\mathrm{d}S/\mathrm{d}t$ \\
\midrule
$10^{0}$  & $5.58\!\times\!10^{-2}$ & $5.58\!\times\!10^{-2}$ & $1.29\!\times\!10^{-8}$ & $2.77\!\times\!10^{-12}$ & $8.77\!\times\!10^{-5}$ \\
$10^{-1}$ & $5.58\!\times\!10^{-3}$ & $5.58\!\times\!10^{-2}$ & $1.71\!\times\!10^{-7}$ & $4.03\!\times\!10^{-13}$ & $8.77\!\times\!10^{-7}$ \\
$10^{-2}$ & $5.58\!\times\!10^{-4}$ & $5.58\!\times\!10^{-2}$ & $1.69\!\times\!10^{-5}$ & $3.98\!\times\!10^{-13}$ & $8.77\!\times\!10^{-9}$ \\
$10^{-3}$ & $5.57\!\times\!10^{-5}$ & $5.57\!\times\!10^{-2}$ & $1.65\!\times\!10^{-3}$ & $3.84\!\times\!10^{-13}$ & $8.79\!\times\!10^{-11}$ \\
$10^{-4}$ & $5.48\!\times\!10^{-6}$ & $5.48\!\times\!10^{-2}$ & $5.05\!\times\!10^{-2}$ & $9.67\!\times\!10^{-14}$ & $9.32\!\times\!10^{-13}$ \\
$10^{-5}$ & $5.48\!\times\!10^{-7}$ & $5.48\!\times\!10^{-2}$ & $1.19\!\times\!10^{-1}$ & $3.80\!\times\!10^{-15}$ & $8.95\!\times\!10^{-15}$ \\
$10^{-6}$ & $5.48\!\times\!10^{-8}$ & $5.48\!\times\!10^{-2}$ & $1.22\!\times\!10^{-1}$ & $4.18\!\times\!10^{-17}$ & $8.94\!\times\!10^{-17}$ \\
$10^{-7}$ & $5.48\!\times\!10^{-9}$ & $5.48\!\times\!10^{-2}$ & $1.22\!\times\!10^{-1}$ & $4.18\!\times\!10^{-19}$ & $8.94\!\times\!10^{-19}$ \\
$10^{-9}$ & $5.48\!\times\!10^{-11}$ & $5.48\!\times\!10^{-2}$ & $1.22\!\times\!10^{-1}$ & $4.18\!\times\!10^{-23}$ & $8.94\!\times\!10^{-23}$ \\
\bottomrule
\end{tabular}
\end{table}

\paragraph{(c) Near-equilibrium states.} Table~\ref{tab:stab-c} evaluates the trained model at $u=A\,u_0$ as $A\to 0$, where $u=0$ is the true fixed point. The structure holds at every amplitude: $\max r_E=1.5\times10^{-8}$, $\min\,\mathrm{d}S/\mathrm{d}t\ge0$, all finite. The gradient and update norms, however, grow as $A\to 0$. This is not due to the projections; rather, it stems from the instance normalization in the functional backbones, which rescales activations by their standard deviation and renders the learned functionals nearly scale-invariant in $u$. Consequently, $\lVert\delta E/\delta u\rVert$ scales as $1/\lVert u\rVert$, capped by the InstanceNorm $\varepsilon=10^{-5}$. Probe~(e) confirms the source.

\begin{table}[ht]
\centering
\caption{\textbf{(c) Near-equilibrium sweep} on the trained heat model,
$u=A\,u_0$. Structure holds at every amplitude; the growth of
$\lVert\delta E\rVert,\lVert\delta S\rVert,\lVert\text{update}\rVert$ as
$A\to0$ is the normalization effect diagnosed in (e).}
\label{tab:stab-c}
\begin{tabular}{lcccccc}
\toprule
$A$ & $\lVert\delta E\rVert$ & $\lVert\delta S\rVert$ & $\lVert\text{update}\rVert$ & $r_E$ & $|\langle\delta E,M\delta S\rangle|$ & $\min\,\mathrm{d}S/\mathrm{d}t$ \\
\midrule
$10^{0}$  & $1.20\!\times\!10^{-2}$ & $5.78\!\times\!10^{-1}$ & $5.58\!\times\!10^{-2}$ & $1.29\!\times\!10^{-8}$ & $2.77\!\times\!10^{-12}$ & $8.77\!\times\!10^{-5}$ \\
$10^{-1}$ & $1.20\!\times\!10^{-1}$ & $5.73\!\times\!10^{0}$  & $5.61\!\times\!10^{-1}$ & $1.20\!\times\!10^{-8}$ & $1.94\!\times\!10^{-10}$ & $9.46\!\times\!10^{-3}$ \\
$10^{-2}$ & $9.89\!\times\!10^{-1}$ & $4.48\!\times\!10^{1}$  & $4.26\!\times\!10^{0}$  & $1.17\!\times\!10^{-8}$ & $6.85\!\times\!10^{-8}$  & $5.39\!\times\!10^{-1}$ \\
$10^{-3}$ & $9.49\!\times\!10^{0}$  & $4.02\!\times\!10^{2}$  & $4.04\!\times\!10^{1}$  & $1.50\!\times\!10^{-8}$ & $6.54\!\times\!10^{-6}$  & $7.96\!\times\!10^{1}$ \\
$10^{-4}$ & $1.17\!\times\!10^{2}$  & $6.31\!\times\!10^{3}$  & $3.77\!\times\!10^{2}$  & $7.31\!\times\!10^{-9}$ & $9.19\!\times\!10^{-4}$  & $2.62\!\times\!10^{4}$ \\
$10^{-5}$ & $1.36\!\times\!10^{3}$  & $5.03\!\times\!10^{4}$  & $1.81\!\times\!10^{3}$  & $9.58\!\times\!10^{-9}$ & $3.27\!\times\!10^{-2}$  & $1.23\!\times\!10^{6}$ \\
$10^{-6}$ & $2.29\!\times\!10^{3}$  & $3.01\!\times\!10^{5}$  & $7.42\!\times\!10^{3}$  & $1.20\!\times\!10^{-8}$ & $1.71\!\times\!10^{-1}$  & $7.52\!\times\!10^{6}$ \\
\bottomrule
\end{tabular}
\end{table}

\paragraph{(e) Diagnosis: InstanceNorm, not the projections.}

The cause is isolated by repeating the amplitude sweep on a fresh randomly initialized model, with and without InstanceNorm in the functional nets (Table~\ref{tab:stab-e}; training is irrelevant to this scaling). In the absence of InstanceNorm, the gradient stays constant in \(A\) and the update shrinks with the field, as a gradient flow should; with InstanceNorm present, both the gradient and the update grow as \(1/A\). Since the rank-one projections are identical across the two columns, they are never the source of the behavior. Consequently, if near-equilibrium accuracy matters, the fix is to remove InstanceNorm from the functional nets—this is also the diagnosed cause of the long-time heat under-dissipation reported in \S\ref{sec:longhorizon} (Appendix~\ref{app:underdiss}).

\begin{table}[ht]
\centering
\caption{\textbf{(e) InstanceNorm diagnosis} on a fresh random-init heat model.
With InstanceNorm, $\lVert\delta E\rVert$ and the update grow as $1/A$; without
it they stay bounded and the update shrinks with the field. The projections are
identical in both.}
\label{tab:stab-e}
\begin{tabular}{lcccc}
\toprule
$A$ & $\lVert\delta E\rVert$ (IN) & $\lVert\delta E\rVert$ (no IN) & $\lVert\text{upd}\rVert$ (IN) & $\lVert\text{upd}\rVert$ (no IN) \\
\midrule
$10^{0}$  & $2.57\!\times\!10^{-3}$ & $1.44\!\times\!10^{-5}$ & $5.66\!\times\!10^{-4}$ & $2.75\!\times\!10^{-6}$ \\
$10^{-1}$ & $2.22\!\times\!10^{-2}$ & $1.43\!\times\!10^{-5}$ & $6.03\!\times\!10^{-3}$ & $3.21\!\times\!10^{-7}$ \\
$10^{-2}$ & $2.38\!\times\!10^{-1}$ & $1.43\!\times\!10^{-5}$ & $7.30\!\times\!10^{-2}$ & $3.23\!\times\!10^{-8}$ \\
$10^{-3}$ & $1.56\!\times\!10^{0}$  & $1.43\!\times\!10^{-5}$ & $3.98\!\times\!10^{-1}$ & $3.36\!\times\!10^{-9}$ \\
$10^{-4}$ & $5.10\!\times\!10^{0}$  & $1.43\!\times\!10^{-5}$ & $1.62\!\times\!10^{0}$  & $8.01\!\times\!10^{-10}$ \\
$10^{-5}$ & $2.02\!\times\!10^{1}$  & $1.43\!\times\!10^{-5}$ & $9.73\!\times\!10^{0}$  & $3.76\!\times\!10^{-10}$ \\
$10^{-6}$ & $1.51\!\times\!10^{0}$  & $1.43\!\times\!10^{-5}$ & $2.09\!\times\!10^{1}$  & $3.38\!\times\!10^{-10}$ \\
\bottomrule
\end{tabular}
\end{table}

\paragraph{(d) Long rollout toward equilibrium.} Over a 2000-step heat rollout on the trained model, no NaN or Inf appeared at any step. Throughout the decay of the field, the update norm holds steady at \(O(10^{-2})\) and \(r_E\) remains near \(10^{-8}\). The physical energy \(\tfrac12\lVert u\rVert^2\) falls during the first few hundred steps and then rises slightly—reflecting the under-dissipation discussed in \S\ref{sec:longhorizon}—while all structural identities stay intact. Consequently, the projections remain well-defined and bounded over the full trajectory.

\medskip
\noindent\emph{Summary.}  As contractions (\(\lVert I-P_w\rVert\le 1\)), the rank-one projections shrink smoothly toward the identity when their direction vanishes, and they deliver finite, bounded updates across fourteen decades of gradient norm, in-model down to \(\tau=10^{-9}\), and through a 2000-step rollout. The projections thus introduce no instability. Near equilibrium, the sole amplitude-dependent scale effect stems from the InstanceNorm in the functional networks—an architectural choice unrelated to the construction itself—which we flag as a limitation (\S\ref{sec:limitations}) and examine in Appendix~\ref{app:underdiss}.

\subsection{Instance normalization in the functional networks: necessity and cost}
\label{app:underdiss}

According to probes (c) and (e) of Appendix~\ref{app:stability}, the instance normalization in the \(E,S\) backbones forces \(\lVert\delta E/\delta u\rVert\) to scale as \(1/\lVert u\rVert\). To quantify both the benefit and the drawback of this choice, this appendix uses a single ablation: GENERIC-noIN, which matches GENERIC-FNO except that every InstanceNorm in \(E_\theta,S_\phi\) is replaced by the identity at initialization (with operators, projections, and heads unchanged), and which is trained under the same protocol as the model of record.

\paragraph{Necessity at 2D resolution.} Training fails for the no-IN model at \(128\times128\). Its loss stays pinned at the initial value, the per-step change in the learned functionals is below single precision (\(\Delta E=\Delta S=0\) to \(10^{-8}\)), and on advection the energy-tracking error is exactly zero—the hallmark of the identity map \(u_{t+1}=u_t\), which conserves \(Q\) exactly on a reversible flow. Its rollout errors (\(0.095\) heat, \(0.170\) advection, \(0.017\) Burgers; three seeds) coincide with the persistence values in Table~\ref{tab:acc-2d}. The reason is the \(\sim\!200\times\) smaller gradient scale at initialization (probe (e)): the optimizer never escapes the do-nothing basin. Hence, normalization is indispensable for the 2D functional networks to learn.

\paragraph{Cost, measured where training succeeds.} In 1D, the networks are small enough to train without normalization, which isolates the artifact's cost. Table~\ref{tab:noin-1d} presents five seeds under the protocol of Appendix~\ref{app:1dfno}. Removing IN improves GENERIC-FNO on every PDE in every seed—by \(45\times\) on heat, \(70\times\) on advection, and \(4\times\) on Burgers—and cuts heat monotonicity violations from \(21\%\) of steps to \(0\%\). The reversible diagnostic also sharpens: \(r_{\mathrm{mech}}\) on advection is \(0.008\pm0.007\), with \(\rho_M\le0.03\). Table~\ref{tab:underdiss} confirms the same trend over 200 steps: the no-IN model matches the exact heat decay to within \(1\%\) at every checkpoint, whereas the model of record plateaus and drifts above its initial energy.

\begin{table}[ht]
\centering
\caption{\textbf{1D-FNO normalization ablation} (rollout $L^2$, five seeds,
repo protocol). GENERIC-noIN is better than GENERIC-FNO in five of five seeds
on every PDE.}
\label{tab:noin-1d}
\begin{tabular}{lcccc}
\toprule
PDE & FNO & EP-FNO & GENERIC-FNO & GENERIC-noIN \\
\midrule
Heat      & $0.082\pm0.012$ & $0.066\pm0.009$ & $0.081\pm0.019$ & $\mathbf{0.002\pm0.000}$ \\
Advection & $0.160\pm0.023$ & $0.162\pm0.019$ & $0.362\pm0.069$ & $\mathbf{0.005\pm0.001}$ \\
Burgers   & $0.070\pm0.033$ & $0.066\pm0.023$ & $0.042\pm0.015$ & $\mathbf{0.011\pm0.006}$ \\
Wave      & $0.490\pm0.050$ & $0.490\pm0.041$ & $0.479\pm0.045$ & $\mathbf{0.396\pm0.029}$ \\
\midrule
Heat, $Q$-increasing steps (\%) & 9.3 & 8.2 & 21.0 & 0.0 \\
\bottomrule
\end{tabular}
\end{table}

\begin{table}[ht]
\centering
\caption{\textbf{1D heat, $200$-step rollout} (single seed; training horizon
$20$). Physical energy $Q/Q_0$ against the exact solution. The model of record
conserves the learned $E_\theta$ at every step yet lets $Q$ rise above its
initial value; the no-IN model tracks the truth.}
\label{tab:underdiss}
\begin{tabular}{lccccccc}
\toprule
Model & $L^2$@20 & $L^2$@200 & $Q/Q_0$@20 & @50 & @100 & @200 & $Q$-inc.\ (\%) \\
\midrule
Ground truth & -- & -- & 0.748 & 0.546 & 0.396 & 0.281 & 0.0 \\
FNO          & 0.190 & 4.192 & 0.711 & 0.547 & 0.601 & 1.306 & 38.7 \\
EP-FNO       & 0.190 & 3.604 & 0.708 & 0.537 & 0.585 & 1.498 & 37.9 \\
GENERIC-FNO  & 0.250 & 3.337 & 0.901 & 0.890 & 1.020 & 1.117 & 44.9 \\
GENERIC-noIN & \textbf{0.004} & \textbf{0.052} & 0.749 & 0.548 & 0.399 & 0.285 & 0.0 \\
\bottomrule
\end{tabular}
\end{table}

One 1D no-IN run (Burgers, seed 4) reports \(r_E=0.3\), with an absolute energy drift of \(8\times10^{-10}\) per step. This places the run in the \(\varepsilon\)-floor regime of Appendix~\ref{app:stability}(b), where \(\lVert\delta E/\delta u\rVert^2\) falls below \(10^{-12}\). In that regime the normalized ratio is undefined, but the update itself is negligible, so this is not a degeneracy violation.

\paragraph{Reading.}
Normalization serves two linked purposes: it makes the 2D functionals trainable, and it caps their accuracy near equilibrium—both effects stem from the same mechanism. The ablation quantifies the cost of the current normalization choice and points toward the remedy, namely a normalization that supplies a usable gradient scale at initialization without rendering the functionals scale-invariant in \(u\). No such normalization is proposed here, and the model of record used throughout the paper remains unchanged.

\section{Parameter-Matched DeepONet Ablation}
\label{app:donmatch}

The accuracy gains reported for DeepONet in Section \ref{sec:accuracy} (Table \ref{tab:acc-don}) were achieved by a GENERIC-DeepONet variant that carried an additional $E,S$ branch, giving it roughly three times the parameter count of the baseline ($175$K versus $58$K). To exclude model capacity as the source of the improvement, both architectures were retrained under a single protocol at two matched budgets, approximately $58$K and $175$K, with matching performed in both directions: GENERIC-DeepONet was shrunk to the baseline's budget, and the baseline was grown to GENERIC-DeepONet's budget. Parameter counts were aligned to within about $2\%$ (baseline: $58{,}113$ and $173{,}409$; GENERIC: $56{,}790$ and $174{,}582$). Table \ref{tab:donmatch} shows that GENERIC-DeepONet produced lower $10$-step rollout error at every PDE and both budgets, winning all $6/6$ matched comparisons. Shrinking GENERIC-DeepONet to $58$K barely altered its error, whereas enlarging the baseline to $175$K did not close the gap and slightly worsened it on Burgers. At both budgets, the unconstrained DeepONet's error matched the persistence baseline (Table \ref{tab:acc-don}), meaning it failed to learn field advancement at any size tested. The advantage is thus structural, not a capacity artifact.

\begin{table}[ht]
\centering
\caption{\textbf{Parameter-matched DeepONet ablation} ($10$-step rollout $L^2$,
single matched run per cell; lower is better, best per budget in bold). Both models
are trained under the same protocol at two matched parameter budgets. GENERIC-DeepONet
wins at every PDE and both budgets, and the baseline does not catch up when grown to
$175$K, so its advantage reflects the structure rather than the parameter count. The
gaps far exceed the five-seed variation in Table~\ref{tab:acc-don}.}
\label{tab:donmatch}
\begin{tabular}{lcccc}
\toprule
 & \multicolumn{2}{c}{$\approx$58K params} & \multicolumn{2}{c}{$\approx$175K params} \\
\cmidrule(lr){2-3}\cmidrule(lr){4-5}
PDE & DeepONet & GENERIC & DeepONet & GENERIC \\
\midrule
Heat      & 0.055 & \textbf{0.004} & 0.052 & \textbf{0.004} \\
Advection & 0.209 & \textbf{0.010} & 0.190 & \textbf{0.009} \\
Burgers   & 0.024 & \textbf{0.009} & 0.037 & \textbf{0.015} \\
\bottomrule
\end{tabular}
\end{table}

\section{Gauge-Dependent Mechanism Diagnostics}
\label{app:mechanism}

All thermodynamic conclusions in the main text rest on the gauge-invariant rate \(r_{\mathrm{mech}}\) (Section~\ref{sec:gauge}, Table~\ref{tab:gauge}). Here, for completeness, we introduce the two mechanism diagnostics that operate directly on the learned \((E,S,L,M)\) decomposition, and we clarify why they are relegated to supporting color: each one is gauge-dependent.

The dissipative fraction captures the share of the raw update that passes through the irreversible channel:
\begin{equation}
  \rho_M = \frac{\lVert M\,\delta S/\delta u\rVert}
                {\lVert L\,\delta E/\delta u\rVert + \lVert M\,\delta S/\delta u\rVert} \in [0,1],
\end{equation}
Meanwhile, the normalized learned-entropy production evaluates how closely the entropy gradient aligns with the dynamics, via the cosine of their angle:
\begin{equation}
  r_S = \frac{\langle \delta S/\delta u,\, \partial_t u\rangle}
             {\lVert \delta S/\delta u\rVert\,\lVert \partial_t u\rVert}.
\end{equation}
Because both diagnostics are invariant under rescaling of \(S\), they might seem suitable for reporting as physical observables—but that impression is misleading. Appendix~\ref{app:gauge} (Remark following Proposition~\ref{prop:gauge-inv}) proves why: \(\rho_M\) hinges on the separate contributions \(L\,\delta E/\delta u\) and \(M\,\delta S/\delta u\), which gauge-equivalent tuples can redistribute while preserving their total, and \(r_S\) hinges on the orientation of \(\delta S/\delta u\), which the gauge transformation also shifts.

Empirical evidence confirms this. For Burgers on the 2D FNO backbone, the dissipative fraction \(\rho_M\) ranges from \(0.001\) in one seed to \(1.000\) in another—the identical flow is represented as nearly fully reversible in one instance and nearly fully irreversible in the other—whereas the rollout error (\(\approx 0.028\)) and the gauge-invariant \(r_{\mathrm{mech}}\) (\(\approx 0.008\)) stay essentially constant across those seeds. In other words, the apparent mechanism is an artifact of the gauge choice, whereas the dissipation of physical energy is not. This is precisely the prediction of Proposition~\ref{prop:gauge-inv}, and it explains why the testable thermodynamic statements are confined to the main-text \(r_{\mathrm{mech}}\) table rather than appearing here.

\begin{table}[ht]
\centering
\caption{\textbf{Gauge-\emph{dependent} mechanism diagnostics} $\rho_M$ and $r_S$
per PDE and backbone (five-seed mean~$\pm$~std). These values shift under the
entropy gauge and are reported only to illustrate the learned channel split; see
the gauge-invariant $r_{\mathrm{mech}}$ in Table~\ref{tab:gauge} for the
load-bearing result. The large seed-to-seed variation is itself a symptom of the
gauge freedom: for 2D-FNO Burgers, $\rho_M$ ranges from $0.001$ to $1.000$ across
seeds at essentially unchanged accuracy and $r_{\mathrm{mech}}$.}
\label{tab:mechanism}
\begin{tabular}{llcc}
\toprule
Backbone & PDE & $\rho_M$ & $r_S$ \\
\midrule
\multirow{3}{*}{2D FNO}
 & Heat      & $0.9995\pm0.0002$ & $0.49\pm0.10$ \\
 & Advection & $0.000\pm0.000$ & $0.000\pm0.000$ \\
 & Burgers   & $0.71\pm0.40$ & $0.57\pm0.29$ \\
\midrule
\multirow{3}{*}{1D FNO}
 & Heat      & $0.71\pm0.26$ & $0.33\pm0.08$ \\
 & Advection & $0.04\pm0.03$ & $0.03\pm0.02$ \\
 & Burgers   & $0.41\pm0.30$ & $0.35\pm0.27$ \\
\midrule
\multirow{3}{*}{1D DeepONet}
 & Heat      & $0.91\pm0.06$ & $0.68\pm0.11$ \\
 & Advection & $0.10\pm0.14$ & $0.03\pm0.02$ \\
 & Burgers   & $0.51\pm0.14$ & $0.32\pm0.05$ \\
\bottomrule
\end{tabular}
\end{table}

\section{Training and Reproducibility Details}
\label{app:repro}

\paragraph{Data.} Reference trajectories are generated with a spectral solver from band-limited random initial fields, each constrained to a fixed maximum wavenumber per PDE. The 2D experiments use $150$ training samples with $15$ time steps, while the 1D experiments use $200$ samples with $20$ time steps. The primary 2D benchmark is trained and evaluated at $128\times128$ resolution. For the super-resolution study (§\ref{sec:superres}), training occurs at $64\times64$, with zero-shot evaluation up to $256\times256$ on data band-limited below the coarse-grid Nyquist frequency; because the underlying continuous dynamics are identical across resolutions, the results are directly comparable.

\paragraph{Models and training.} All models are optimized with AdamW (weight decay $10^{-4}$) under a cosine learning-rate schedule. The loss is mean-squared one-step error, augmented by a short autoregressive rollout term (3 steps in 1D, 2 in 2D) applied to a random 25–30\% of batches after a warm-up period, with gradient clipping at 1.0. GENERIC models receive no supervision on $E$ or $S$ and no degeneracy penalty, since those conditions hold by construction. Each run draws 200 (1D) or 150 (2D) trajectories, split 80/20 into training and test sets. One seed is fixed per run, and within a run all models draw their initializations from a single continuing random stream, so models compared under the same seed share data but not initialization. For the headline benchmarks (Tables~\ref{tab:acc-2d}--\ref{tab:acc-1d}), the variational derivative is supplied by the raw autodiff gradient. Because the learned multipliers absorb the constant relating this gradient to the $L^2$ gradient at a fixed grid, the two differ only by that constant, to which the projections are invariant. The $L^2$ scaling $N^d/|\Omega|$, by contrast, is what the super-resolution study (\S\ref{sec:superres}) and the 3D verification employ; as Proposition~\ref{prop:resinv} demands, grid consistency of the step requires it (Appendix~\ref{app:setting}). Parameter counts are as follows: the 2D FNO and EP-FNO baselines have 2.10M parameters, GENERIC-FNO has 1.00M; the 1D FNO baselines have 71K, GENERIC has 47K; the DeepONet baseline has 58K, and GENERIC-DeepONet has 175K (due to its second functional branch). The energy and entropy functionals share the backbone family of their respective baseline and add a spatial-mean pooling layer followed by a two-layer MLP head; the operators $L$ and $M$ act on a fixed low-frequency band. Training runs for 120 epochs in 2D and 150 epochs in 1D.

\paragraph{Long-horizon protocol (\S\ref{sec:longhorizon}).}
Each model is trained at $128\times128$ on $96$ trajectories of $15$ steps for $50$ epochs with Adam (learning rate $10^{-3}$, batch $8$), minimizing the relative $L^2$ error accumulated over $K=2$ rolled-out steps; EP-FNO adds its energy penalty with unit weight. rollouts are evaluated on $8$ held-out trajectories integrated for $200$ steps, with error normalized by the initial reference norm, so a prediction that decays to zero scores exactly $1$. FNO-L is an FNO of width $128$ and $12$ layers ($37.8$M parameters), the smallest in a fixed ladder whose measured training-step time under this protocol reaches that of GENERIC-FNO ($385$ vs.\ $303$\,ms on the same GPU); the U-Net is a four-level residual encoder--decoder \citep{ronneberger2015unet} (base width $32$, $1.9$M parameters) with circular padding for the periodic domain. Single run per cell.

\paragraph{Seeds and evaluation.} Every reported value is the mean ± standard deviation over five seeds, with both the sampled data and model initialization resampled per seed. Checkpoints are saved only on the first seed. Accuracy is measured as the relative $L^2$ error of a 10-step autoregressive rollout on held-out trajectories. Structural residuals ($r_E$, degeneracy) and the gauge-invariant $r_{\mathrm{mech}}$ are computed on held-out initial conditions as defined in §\ref{sec:gauge}. With only five seeds, the reported standard deviations are indicative rather than precise variance estimates; we rely on them only where the effect is large relative to the spread—for example, the order-of-magnitude difference in seed variance on 2D heat (Table~\ref{tab:acc-2d})—and the headline accuracy and long-horizon gaps far exceed that spread. 

\paragraph{Computational cost.} The overhead beyond a plain spectral operator comes from reverse-mode autodiff through the learned functionals: each step requires two backward passes to compute $\delta E/\delta u$ and $\delta S/\delta u$. This cost is intrinsic to the method, not an implementation detail. At the $128^2$ configuration (Table~\ref{tab:compute}; batch size 8, single GPU), the deployed forward pass costs approximately $4.4\times$ that of a plain FNO forward, a training step approximately $10\times$ (because the loss backpropagates through the functional-gradient graph, making the step a double-backward, which dominates and does not parallelize on GPU, where a plain FNO step takes only ${\sim}10$~ms), and peak training memory approximately $2.3\times$---all at half the parameters ($1.00$M vs.\ $2.10$M). We report this plainly as the price of exact variational derivatives; disabling \texttt{create\_graph} at inference saves only about $7\%$, since the dominant cost is the functional gradients themselves, not graph retention. Two factors limit the practical impact: training is one-time and the absolute step cost is small (${\sim}0.1$~s), and the projection machinery itself is $O(N)$ and negligible beside the FFTs (\S\ref{sec:discretization}). The functional-gradient cost scales with the number of fields; cheaper $E,S$ gradients (via shared backbones or a fused double-backward) are natural directions for future work.

\begin{table}[ht]
\centering
\caption{\textbf{Computational cost} at the 2D $128\times128$ config (batch~$8$,
single GPU; latencies averaged over 30 iterations after warmup). ``Forward'' is the
deployed inference latency (\texttt{create\_graph=False}); the training step is a
double-backward through the learned functionals. GENERIC-FNO's overhead is real
(forward $\approx\!4.4\times$, step $\approx\!10\times$, peak memory
$\approx\!2.3\times$ a plain FNO), though it uses half the parameters and the
absolute step cost is small ($\approx\!0.1$\,s).}
\label{tab:compute}
\begin{tabular}{lcccc}
\toprule
Model & Params & Forward (ms) & Train step (ms) & Peak mem (MB) \\
\midrule
FNO         & 2{,}102{,}529 & 3.0  & 9.6  & 358 \\
EP-FNO      & 2{,}102{,}529 & 2.8  & 9.8  & 392 \\
GENERIC-FNO & 1{,}001{,}958 & 13.0 & 98.6 & 809 \\
\bottomrule
\end{tabular}
\end{table}

\paragraph{Code.}
Exact architectures, wavenumber bands, learning rates, and batch sizes are fixed in the
accompanying code, available at the repository
\url{https://github.com/jaysulk/GENERIC-FNO}, which reproduces every table
and figure---including the five-seed wrappers, the capacity sweep, the projection-stability
probes, the normalization ablation, the super-resolution, long-horizon (with compute-matched baselines), 3D-verification, and Euler-versus-RK4 studies. 

\section{Supplementary Figures}
\label{app:figures}

This appendix collects five figures that visualize, rather than re-tabulate, the
exactness, gauge, ordering, integrator, and learned-functional results.
All are generated by the released script from the same five-seed runs.

\begin{figure}[H]
\centering
\includegraphics[width=0.82\linewidth]{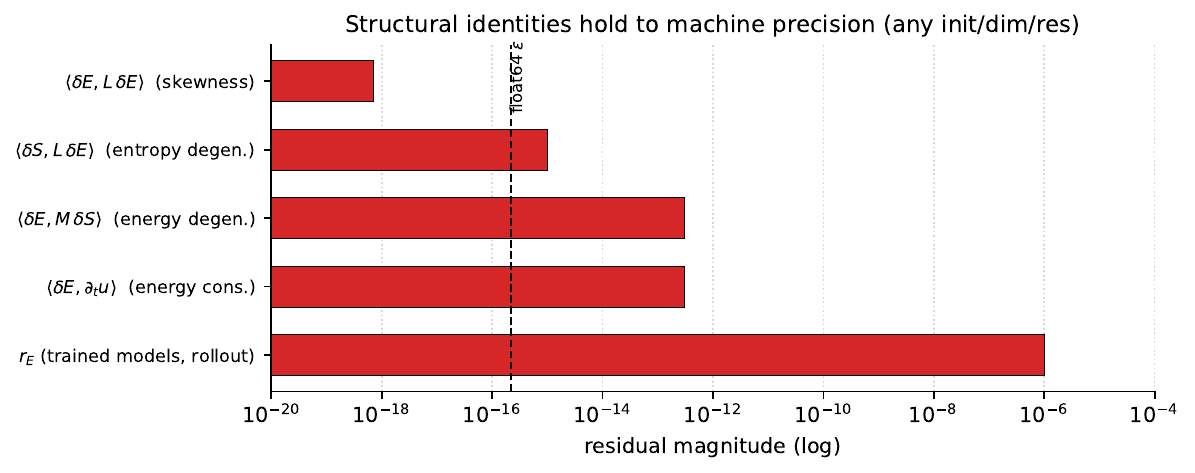}
\caption{\textbf{The structural identities hold to machine precision.}
Magnitude (log scale) of the GENERIC residuals at random initialization on a
$16\times16$ grid: reversible skewness $\langle\delta E,L\,\delta E\rangle$, the
two degeneracy conditions $\langle\delta S,L\,\delta E\rangle$ and
$\langle\delta E,M\,\delta S\rangle$, energy conservation
$\langle\delta E,\partial_t u\rangle$, and the trained-model energy residual
$r_E$. All sit at or near the float64 machine epsilon (dashed line), confirming
that degeneracy is enforced by construction rather than approximately
(\S\ref{sec:guarantees}, Appendix~\ref{app:proofs}). Only $r_E$, which also
reflects the explicit-Euler step, rises to $\sim\!10^{-6}$.}
\label{fig:degeneracy}
\end{figure}

\begin{figure}[H]
\centering
\includegraphics[width=\linewidth]{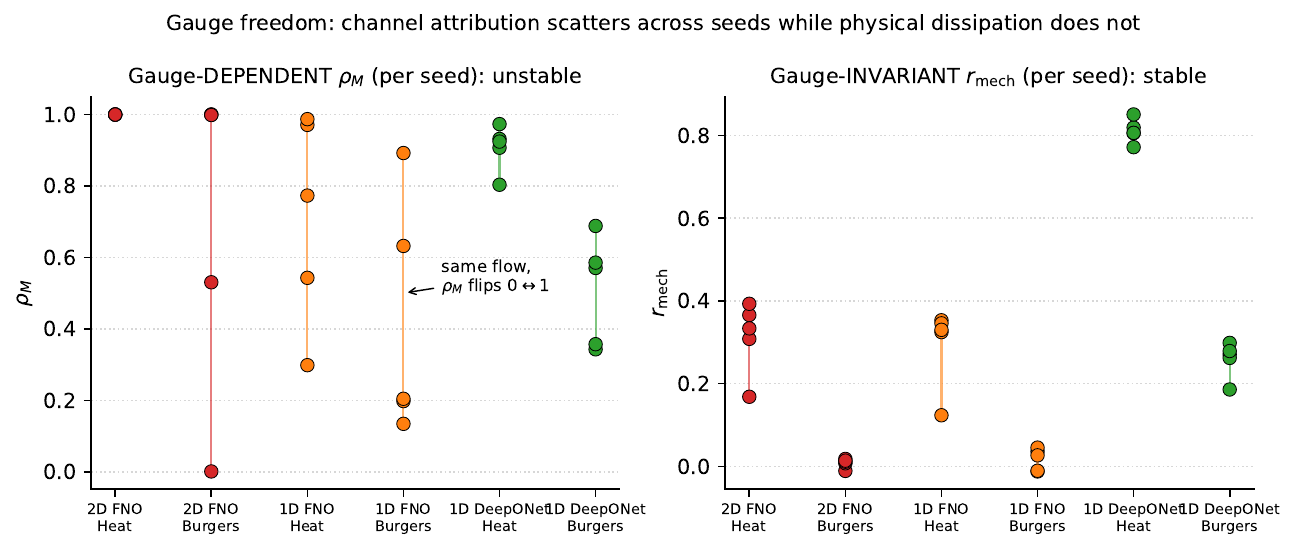}
\caption{\textbf{Gauge freedom, made visual.} Per-seed values of the
gauge-\emph{dependent} dissipative fraction $\rho_M$ (left) and the
gauge-\emph{invariant} dissipation rate $r_{\mathrm{mech}}$ (right), for the
non-reversible PDEs across all three backbones. The channel attribution $\rho_M$
scatters wildly across seeds---for 2D-FNO Burgers it spans the full range from
$0.001$ to $1.000$, i.e.\ the model represents the \emph{same} flow once as almost
purely reversible and once as almost purely irreversible---whereas
$r_{\mathrm{mech}}$, computed from the fixed physical energy, stays stable for that
same flow. Since the trivial rescalings of $(E,L)$ and $(S,M)$ leave $\rho_M$
unchanged, these swings show the realized gauge freedom exceeds rescaling; its
extent is discussed, not proved, in Appendix~\ref{app:gauge}, which proves only
that $r_{\mathrm{mech}}$ is invariant under it. This is the reason thermodynamic
claims rest on $r_{\mathrm{mech}}$ (\S\ref{sec:gauge}, Appendix~\ref{app:mechanism}).}
\label{fig:gauge_variance}
\end{figure}

\begin{figure}[H]
\centering
\includegraphics[width=\linewidth]{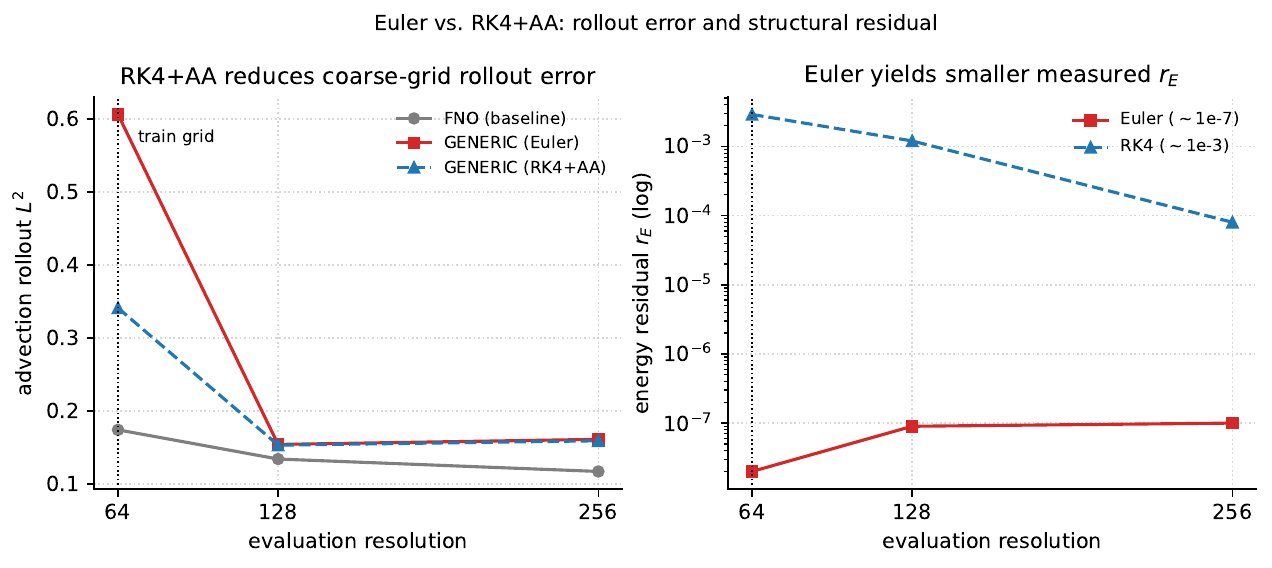}
\caption{\textbf{Why we keep explicit Euler.} Advection trained at $64$ and
evaluated to $256$. \emph{Left:} rollout $L^2$ for the FNO baseline, GENERIC with
explicit Euler, and GENERIC with RK4 + anti-aliasing. RK4 roughly halves the
coarse-grid blow-up (Euler $0.61\to$ RK4 $0.34$ at the train grid) and the two
integrators agree at finer grids. \emph{Right:} the per-step energy residual
$r_E$, which RK4 raises from $\sim\!10^{-7}$ (Euler)
to $\sim\!10^{-3}$. We retain Euler because it gives the smaller measured
energy residual (\S\ref{sec:limitations}).}
\label{fig:euler_rk4}
\end{figure}

\begin{figure}[H]
\centering
\includegraphics[width=0.6\linewidth]{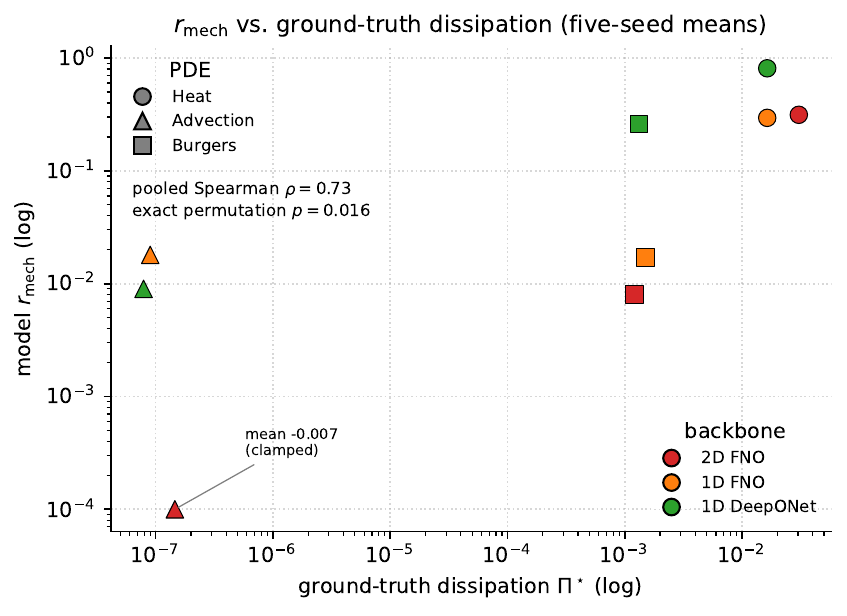}
\caption{\textbf{$r_{\mathrm{mech}}$ against the ground-truth dissipation.}
Each point is one (backbone, PDE) pair: the model's gauge-invariant rate
$r_{\mathrm{mech}}$ against the ground-truth rate $\Pi^\star$ (both log scale;
markers denote PDE, colors denote backbone; five-seed means). The reversible
advection points cluster at the bottom-left ($\approx 0$; the 2D-FNO advection
mean of $-0.007$ is clamped to the axis floor) and heat sits top-right in every
backbone. Within backbones the ordering is exact on the 2D FNO and DeepONet
(Spearman $\rho=1$) and not on the 1D FNO ($\rho=0.5$, Burgers and advection
within noise); pooled over all nine points $\rho=0.73$ with exact permutation
$p=0.016$ (\S\ref{sec:thermo}). The comparison is ordinal: $r_{\mathrm{mech}}$
and $\Pi^\star$ use different normalizations, so only the ranking---not the
absolute offset between the axes---is meaningful.}
\label{fig:rmech_ordering}
\end{figure}

\begin{figure}[H]
\centering
\includegraphics[width=\linewidth]{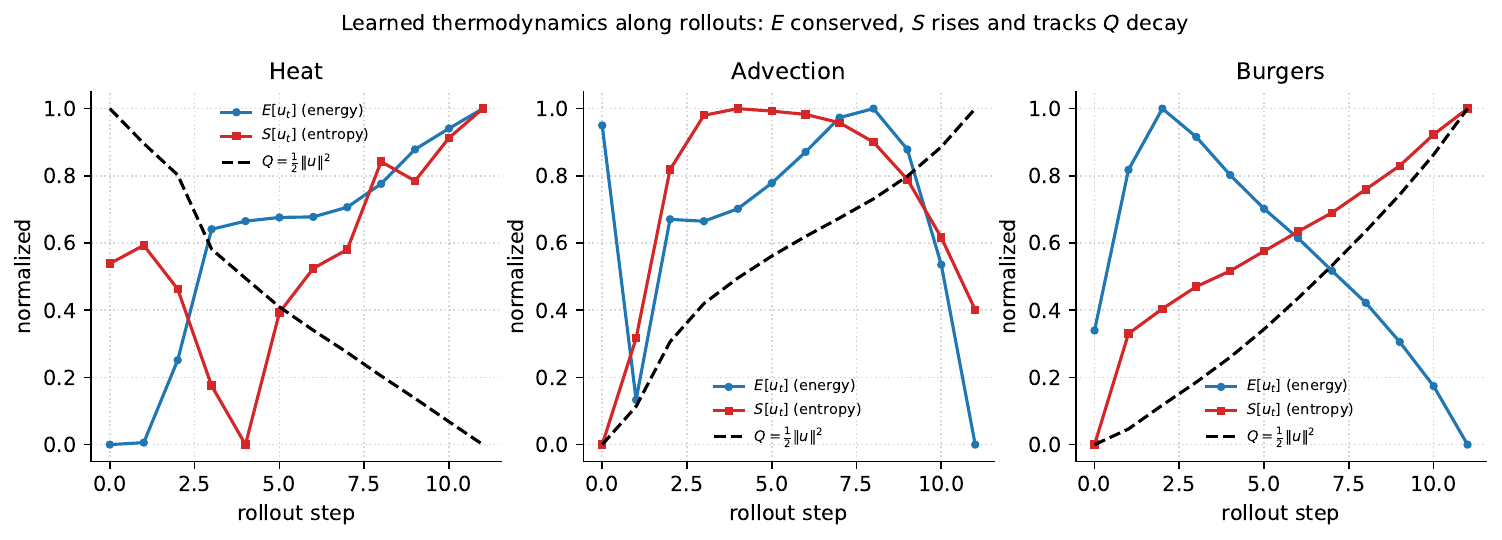}
\caption{\textbf{Learned thermodynamics along a rollout (illustrative, single
realization).} For one trained model per PDE we plot the learned energy
$E_\theta[u_t]$, the learned entropy $S_\phi[u_t]$, and the fixed mechanical
energy $Q=\tfrac12\|u_t\|^2$ along a trajectory, each affinely rescaled to $[0,1]$
within its panel for display. In this realization $E_\theta$ stays flat
(conserved) while $S_\phi$ rises monotonically and tracks the decay of $Q$ for the
dissipative and mixed PDEs, and both are flat for reversible advection---so the
unsupervised functionals behave like a physical energy and entropy. We show this
only as an interpretive aid: the absolute values, scale, and per-channel
attribution of $E_\theta,S_\phi$ are gauge-dependent (\S\ref{sec:gauge}), so the
figure is \emph{not} a gauge-invariant claim. The falsifiable, gauge-invariant
statements are the dissipation diagnostics of Table~\ref{tab:gauge}
(\S\ref{sec:thermo}).}
\label{fig:interpretability}
\end{figure}

\begin{figure}[H]
\centering
\includegraphics[width=0.85\linewidth]{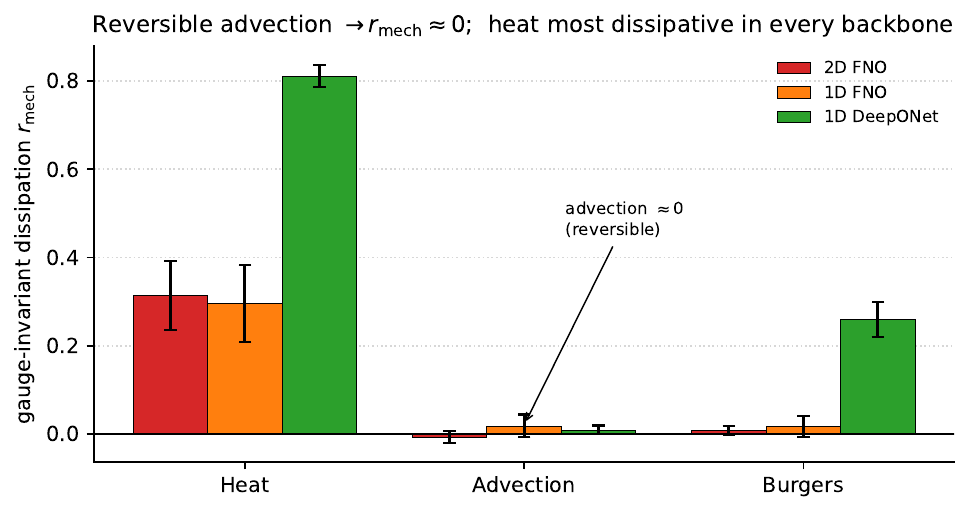}
\caption{\textbf{Gauge-invariant dissipation $r_{\mathrm{mech}}$
(five-seed mean~$\pm$~std) per PDE, grouped by backbone.} Advection sits at
$\approx 0$ ($|r_{\mathrm{mech}}|\le0.035$ in every seed) and heat is clearly
positive in every backbone; Burgers is separated from advection on DeepONet and
the 2D FNO but not on the 1D FNO. Computed from a fixed quadratic energy, so it
is independent of the learned (gauge-dependent) $E,S$.}
\label{fig:dissipation}
\end{figure}

\end{document}